\documentclass[11pt]{article}

\usepackage[preprint]{acl}

\usepackage{times}
\usepackage{latexsym}

\usepackage[T1]{fontenc}

\usepackage[utf8]{inputenc}

\usepackage{microtype}

\usepackage{inconsolata}

\usepackage{algorithm}
\usepackage{algorithmic}
\usepackage{caption}
\usepackage{graphicx}
\usepackage{amsmath}
\usepackage{amssymb} 
\usepackage{booktabs}
\usepackage{multirow}
\usepackage[most]{tcolorbox}
\usepackage{enumitem}
\usepackage{makecell}
\usepackage{tabularx}
\usepackage{amsthm}
\usepackage{placeins}
\usepackage{stfloats}

\newlist{scorelist}{enumerate}{1}
\setlist[scorelist]{
  label=\textbf{\arabic*:},
  leftmargin=2em,
  labelsep=0.4em,
  itemsep=0pt,
  parsep=0pt,
  partopsep=0pt,
  topsep=0.2em
}

\newtheorem{assumption}{Assumption}
\newtheorem{theorem}{Theorem}
\newtheorem{lemma}{Lemma}

\newcommand{\KL}{D_{\mathrm{KL}}}
\newcommand{\ModelName}{{ToSCA}}

\newtcolorbox{promptbox}[1]{
  enhanced jigsaw,
  breakable,
  colback=white,
  colframe=black!65,
  colbacktitle=white,
  coltitle=black,
  fonttitle=\bfseries\small,
  fontupper=\footnotesize,
  boxrule=0.45pt,
  arc=0pt,
  left=1.5mm,
  right=1.5mm,
  top=0.7mm,
  bottom=0.7mm,
  before skip=3pt,
  after skip=5pt,
  title=#1
}

\title{ToSCA: Leveraging Hierarchical Reinforcement Learning on Temporal and Strategic Abstractions of Conversational Agents}

\author{
 \textbf{Xiaoyu Wang\textsuperscript{1,2,\thanks{This work is completed during their internships at Geely.},\footnotemark[2]}},
 \textbf{Qingqing Gu\textsuperscript{1,\thanks{The first two authors contribute equally.}}},
 \textbf{Yue Zhao\textsuperscript{1}},
 \textbf{Teng Chen\textsuperscript{1}},
 \\
 \textbf{Yuqi Cao\textsuperscript{1,3,\footnotemark[1]}},
 \textbf{Xiaokai Chen\textsuperscript{2}},
 \textbf{Hongyan Li\textsuperscript{1}},
 \textbf{Luo Ji\textsuperscript{1}},
\\
\\
 \textsuperscript{1}Geely AI Lab,
 \textsuperscript{2}Beijing Institute of Technology,
 \textsuperscript{3}Peking University,
\\
 \small{
   \textbf{Correspondence:} \href{Luo.Ji1@geely.com}{jiluoaaron@hotmail.com} 
 }
}

\begin{document}
\maketitle
\begin{abstract}
Humans naturally exhibit multiple forms of abstraction in reasoning and interaction, including temporal abstraction across decision timescales and strategic abstraction over communicative intents. Inspired by these complementary abstractions, we propose a two-level hierarchical reinforcement learning (HRL) framework for conversational agents that bridges the gap between existing token-level and utterance-level RL methods. Built upon a two-level Markov decision process (MDP), our framework conditions token-level response generation on utterance-level actions represented by explicit textual strategies. Based on theoretical analysis and efficiency considerations, we employ DQN to optimize the high-level Q-network and PPO to train the low-level actor-critic. To further alleviate reward sparsity and facilitate convergence, we introduce a dual-granularity reward mechanism that combines the utterance-level satisfaction score with token-level intrinsic self-consistency and a KL-divergence penalty. Experiments on both daily-life and emotional support conversations demonstrate that our method consistently outperforms a wide range of baselines in both strategy determination and response quality. Our implementation is available at \url{https://github.com/AaronJi/ToSCA}.
\end{abstract}

\section{Introduction}

Large Language Models (LLMs) have achieved remarkable progress in coding, reasoning, and conversational generating  \cite{Ouyang2022ChatGPT,llama3modelcard}. Nevertheless, conversational modeling still poses several challenges, including i) \emph{the limited availability of supervised signals}, ii) alignment with \emph{nuanced and complex human preferences}, and iii) \emph{the broad and diverse distribution of conversational topics}. These factors limit the adaptability and generalization of conventional supervised approaches (\textit{e.g.}, SFT or LoRA) \cite{hedayatniaSystematicEvaluationResponse2022, yueDoesReinforcementLearning2025, leeModelingOnetoManyProperty2026}. Reinforcement learning from human feedback (RLHF) provides a natural alternative by enabling policy exploration beyond fixed demonstrations and direct optimization toward long-term conversational returns, with representative methods including Proximal Policy Optimization (PPO) \cite{Ouyang2022ChatGPT} and Group Relative Policy Optimization (GRPO) \cite{shaoDeepSeekMathPushingLimits2024}.

\begin{figure}
    \centering
    \includegraphics[width=0.99\linewidth]{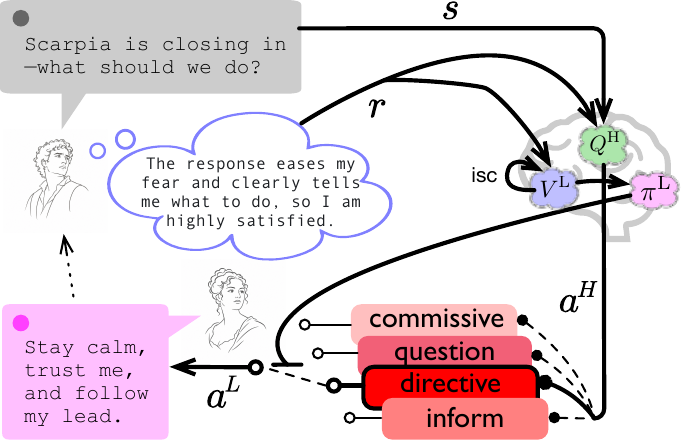}
    \caption{The paradigm of {\ModelName} with two-level MDPs. The high-level Q-network $Q^\text{H}$ considers the user query ($s$), bootstraps the user's satisfaction reward ($r$), and determines the optimal strategy ($a^\text{H}$). The low-level critic $V^\text{L}$ considers both $r$ and the intrinsic self-consistency (isc), guiding the low-level actor $\pi^\text{L}$. The \textit{switch} $\bot$ over the \textit{contacts} $\multimap$ denotes the `options' mechanism in which the discrete high-level action conditions the low-level action ($a^\text{L}$), as the detailed response.}
    \label{fig:paradigm}
\end{figure}


Despite their success, most existing reinforcement learning (RL) approaches operate at the token level \cite{jaques-etal-2020-human,Ouyang2022ChatGPT}, which can lead to sparse rewards and substantial computational overhead \cite{pmlr-v235-zhou24t}. These challenges become particularly pronounced in multi-turn conversations, where \emph{meaningful user feedback is often unavailable until the assistant completes an entire utterance} \cite{liDialogueActionTokens2024}. As a result, exploration over the large token-level action space becomes prohibitively expensive, making it difficult for LLM-based agents to learn effective and human-like conversational behaviors.


One possible remedy is inspired by the hierarchical nature of human cognition \cite{Murray2014hierarchyofintrinsictimescales,Zeraati2023Intrinsictimescalesinthevisualcortex}. In conversation, humans often first determine a high-level communicative strategy, such as \emph{commissive}, \emph{directive}, \emph{question}, or \emph{inform}, then formulate a concrete response accordingly. Such strategies have been systematically studied and annotated in conversation corpora based on sociological or psychological theories, such as the Helping Skills Theory \cite{Hill2009Helping} in ESConv \cite{liu2021ESconv}, and the ISO 24617-2 annotation scheme \cite{conf/lrec/PetukhovaMB14} in DailyDialog \cite{li-etal-2017-dailydialog}. This abstraction mechanism motivates a variety of strategy-aware approaches, including prompting methods \cite{10.1007/978-981-92-1926-1_10}, supervised learning \cite{10.5555/3666122.3666237,Liu2024APR}, retrieval frameworks \cite{yang2024BOT}, and simulation-based studies \cite{zhou-etal-2024-think}. In comparison, RL has been less explored for explicitly modeling the hierarchy between strategic decisions and concrete token-level generation.

Hierarchical Reinforcement Learning (HRL) \cite{10.1016/S0004-3702(99)00052-1, mcgovern01} provides a natural framework for such hierarchical decision-making, and has addressed analogous challenges in robotics and control through temporal abstraction and action-space reduction \cite{shiHiRobotOpenEnded2025,sunHierarchicalReinforcementLearning2025,watanabeHierarchicalReinforcementLearning2025}. However, directly transferring conventional HRL methods to human conversation remains challenging due to fundamental differences in state and action representations. Recent RL methods therefore adopt different forms of abstraction, but each leaves part of this hierarchy under-modeled. Utterance-level approaches such as straQ* \citep{wang-etal-2025-convert}, DAT \cite{li2024dialogue}, and hierarchical action exploration \cite{cho-etal-2024-deep} introduce high-level planning or coarse-grained actions, but do not directly optimize token-level generation policies. In contrast, ArCHer \cite{pmlr-v235-zhou24t} incorporates both utterance-level and token-level optimization within an HRL framework, but represents its high-level intents in a continuous latent space without explicit textual semantics, limiting interpretability and explicit alignment between human-readable strategies and token-level actions.
To jointly realize explicit strategic abstraction and token-level policy optimization, we propose a novel dual-level HRL framework called \textbf{T}emp\textbf{o}ral-\textbf{S}trategic Abstractions of \textbf{C}onversational \textbf{A}gents (\textbf{{\ModelName}}), which integrates an utterance-level strategic planner with token-level policy optimization. {\ModelName} follows the \textbf{options} framework \cite{10.1016/S0004-3702(99)00052-1,stolle03}, where low-level actions (response tokens) are conditioned on high-level actions (strategies), as illustrated in Figure \ref{fig:paradigm}. Based on a two-level MDP formulation, we employ a high-level critic over a discrete strategy space, optimized with DQN, and a low-level actor-critic architecture optimized with PPO. To alleviate reward sparsity and improve hierarchical consistency, we further introduce a dual-granularity reward mechanism, where the low-level reward combines the utterance-level user satisfaction score with two token-level auxiliary components: the K-L penalty and intrinsic self-consistency. We train and evaluate {\ModelName} on two conversational domains, \emph{daily-life conversation} (DailyDialog) and \emph{emotional support conversation} (ESConv), and further conduct an out-of-domain evaluation on EmpatheticDialogues. {\ModelName} consistently outperforms prompting, supervised, and RL baselines, demonstrating strong strategic planning and generalization across domains. Major contributions of this paper include:
\begin{itemize}
\item We propose an HRL-based dialogue framework that bridges explicit utterance-level strategic planning and token-level response generation within a unified architecture.
\item Besides utterance-level satisfaction, we introduce intrinsic self-consistency to form a dual-granularity reward that alleviates reward sparsity and facilitates policy optimization.
\item We conduct extensive experiments on daily-life and emotional support conversations to demonstrate the effectiveness and out-of-domain generalization of our method.
\end{itemize}


\section{Preliminaries}
\label{sec:preliminary}

\subsection{Markov Decision Process} 

The Markov decision process (MDP) is usually defined as a 5-tuple $(\mathcal{S}, \mathcal{A}, \mathcal{R}, \mathcal{T}, \gamma)$, where $\mathcal{S}$ is the state set, $\mathcal{A}$ is the action set, $\mathcal{R}$ is the reward set, $\gamma$ is the discounting factor of rewards, and $\mathcal{T}: \mathcal{S} \times \mathcal{A} \rightarrow \mathcal{S}$ is the state transition function.

At a specific time $t$, Reinforcement Learning (RL) can be employed to provide a policy $\pi(a|s)$ by optimizing the following discounted cumulative expected return ($J$):
\begin{align}
   &\max_{\pi}{ \mathbb{E}_{s, a \sim \mathcal{T}, \pi} \sum_{t=0}^{\infty} \gamma^{t} r_{t} } := J(\mathcal{M}) \label{eq:J} \\
   &s \in \mathcal{S}, a \in \mathcal{A}, r \in \mathcal{R} \notag
\end{align}

\subsection{Deep Q-learning}

Q-Learning belongs to value-based RL, which aims to learn the optimal state-action value function $Q^{*}(s, a)$, such that the determined action approximately optimizes $J$ by maximizing $Q$. Instead of explicitly implementing the above equation, Deep Q-learning (DQN) \cite{Mnih2015DQN} approximates the minimization of the error of the Bellman Equation with deep value networks: 
\begin{align}
    &\mathcal{L}_{Q}(\phi) = \notag \\
    &| r(s,a) + \gamma \max_{a'} Q_{\bar{\phi}}(s', a') - Q_{\phi}(s, a)|^2 \label{eq:dqn}  
\end{align}
where $\phi$ and $\bar{\phi}$ are parameters of Q and target Q-nets. $\bar{\phi}$ is periodically synchronized from $\phi $. 

\subsection{Proximal Policy Optimization} 


Proximal policy optimization (PPO)~\citep{schulman2017proximal} follows the Actor-Critic framework and is widely applied to LLMs~\citep{Ouyang2022ChatGPT}. Initialized from an LLM backbone, the \emph{actor} behaves as a generative policy network, while a separate value network serves as the \emph{critic}. The actor is trained by maximizing the objective
\begin{align}
\mathcal{L}_{\pi}(\theta)
=&\;
\mathbb{E}_{(s,a)\sim\mathcal{T},\,\pi_{\theta_{\text{old}}}} \Big[ \min \big( \frac{\pi_\theta(a|s)}{\pi_{\theta_{\text{old}}}(a|s)}\, A(s,a), \notag \\
&\;
\text{clip}(\frac{\pi_\theta(a|s)}{\pi_{\theta_{\text{old}}}(a|s)}, 1-\epsilon, 1+\epsilon) A(s,a)  \big) \Big] \notag \\
&\;
- \beta_1\,
\mathbb{E}_{s\sim\mathcal{T}}
\Big[ \KL\!\left( \pi_\theta(\cdot|s)\,\Vert\,\pi^{\text{ref}}(\cdot|s) \right) \Big] \notag \\
&\;
+ \beta_2\,
\mathbb{E}_{x\sim\mathcal{D}_{\text{pretrain}}} \Big[ \log \pi_\theta(x) \Big]
\label{eq:rlhf_actor}
\end{align}
where $\theta$ and $\theta_{\text{old}}$ denote the parameters of the current policy and the last policy, respectively; $\epsilon$ is the PPO clipping threshold that constrains excessive policy updates; $A(s,a)$ is the advantage function estimated by the critic; $\pi^{\text{ref}}$ is usually initialized from and fixed as the SFT policy; and $\log \pi_\theta(x)$ is an auxiliary language modeling objective which sums the token-level auto-regressive log-likelihoods. As the sub-objective weights, $\beta_1$ controls the strength of the KL regularization that constrains the policy to remain close to the SFT model, while $\beta_2$ balances the auxiliary pretraining objective. The critic is trained to regress the state-value function using
\begin{align}
\mathcal{L}_{\text{V}}(\psi) = \mathbb{E} \Big[\big(V_{\psi}(s) - A(s,a) - V_{\psi_\text{old}}(s) \big)^2 \Big] \label{eq:rlhf_critic}
\end{align}
where $V_{\psi}$ denotes the state-value estimate, $\psi$ and $\bar \psi$ are its current and previous parameters.

\section{Method}

\subsection{Task Formulation} 

\paragraph{Multi-turn conversation.} A $T$-turn user-agent conversation can be represented by an interleaved sequence of $query(t)$ and $response(t)$, $1 \leq t \leq T$. This paper studies strategy-annotated conversations \cite{liu2021ESconv,rashkin-etal-2019-towards}, with each turn annotated with the user's emotion ($emo$) and the agent's strategy ($stra$). Such an augmented conversation is then described by $\{ desc,\allowbreak hist(T),\allowbreak query(T),\allowbreak
emo(T),\allowbreak stra(T),\allowbreak response(T) \}$, in which $desc$ is the session-level description, and $hist(T) := \{query(t), \allowbreak response(t) \}_{0:T-1}$ is the history context at $t$.

\paragraph{Hierarchical MDP.} Similar to ArCHer \cite{pmlr-v235-zhou24t}, we formulate multi-turn conversation as a hierarchical MDP with temporal abstraction across two decision scales:
\begin{enumerate}
    \item The high-level MDP $\mathcal{M}^\text{H}$: the agent makes strategic decisions at the utterance level and operates on the dialogue-turn timescale indexed by $t$, which guide the low-level actions.
    \item The low-level MDP $\mathcal{M}^\text{L}$: Similar to RLHF, this MDP operates at the token level with token index $k$, where the agent generates a response token by token until completion. 
\end{enumerate}
where the superscripts H and L denote the high and low levels, respectively; $t$ indexes real-time dialogue turns, while $k$ indexes token-level generation steps. Unlike \citet{pmlr-v235-zhou24t}, our high-level actions provide explicit strategic abstraction, which can be learned from and evaluated against strategy-annotated datasets.


\subsection{System Configuration} 

\paragraph{Variable definitions.} According to the environment, we define its system variables as in Table \ref{Tab:define}, where $response_{0:k}$ is the currently generated part of $response$. Based the above definitions, $\mathcal{M}^\text{H}$ and $\mathcal{M}^\text{L}$ can be depicted with tuples of $(\mathcal{S}^{\text{H}}, \mathcal{A}^{\text{H}}, \mathcal{R}^{\text{H}})$ and $(\mathcal{S}^{\text{L}}, \mathcal{A}^{\text{L}}, \mathcal{R}^{\text{L}})$, respectively. 

\begin{table}[t]
\centering
\resizebox{\columnwidth}{!}{%
\begin{tabular}{c l l}
\toprule
Level & Variable & Definition \\
\midrule

\multirow{3}{*}{\rotatebox{90}{\textbf{High}}} 
& $s^{\text{H}} \in \mathcal{S}^{\text{H}}$ 
& \textbf{State} $(query,\; history,\; emo)$ \\

& $a^{\text{H}} \in \mathcal{A}^{\text{H}}$ 
& \textbf{Selected strategy} $(stra)$ \\

& $r^{\text{H}} \in \mathcal{R}^{\text{H}}$ 
& \textbf{Satisfaction score} $(r^{\text{sat}})$ \\

\midrule

\multirow{3}{*}{\rotatebox{90}{\textbf{Low}}} 
& $s^{\text{L}} \in \mathcal{S}^{\text{L}}$ 
& \textbf{State} $(s^{\text{H}},\; response_{0:k})$ \\

& $a^{\text{L}} \in \mathcal{A}^{\text{L}}$ 
& \textbf{Current decoded token} \\

& $r^{\text{L}} \in \mathcal{R}^{\text{L}}$
& \textbf{Token-level reward} \\

\bottomrule
\end{tabular}
}
\caption{Variable definitions in the HRL framework.}
\label{Tab:define}
\end{table}

\paragraph{RL components and objectives.} In this work, we use HRL to jointly optimize the objectives of $\mathcal{M}^\text{H}$ and $\mathcal{M}^\text{L}$. Unlike previous configurations such as Actor$^{\text{H}}$-Critic$^{\text{L}}$ \cite{pmlr-v235-zhou24t}, Actor$^{\text{H}}$-Critic$^{\text{H}}$-Actor$^{\text{L}}$ \cite{10.5555/3298483.3298491}, or double Actor-Critic \cite{NEURIPS2019_4f284803}, our high-level MDP employs a Q-network ($Q^{\text{H}}$) to perform value-based strategic selection over discrete options, while the low-level MDP adopts an actor-critic framework ($\pi^{\text{L}}$, $V^{\text{L}}$), leveraging the robustness of token-level PPO. Accordingly, we refer to our architecture as \textbf{Critic$^{\text{H}}$-Actor$^{\text{L}}$-Critic$^{\text{L}}$}\footnote{The high-level Q-network serves as a value-based critic without an explicit high-level actor.}, with components learned by the following objectives:



\begin{align}
\pi^{\text{H}} \leftarrow \max J(\mathcal{M}^{\text{H}}); \pi^{\text{L}}, Q^{\text{L}} \leftarrow \max J(\mathcal{M}^{\text{L}}) \label{eq:obj}
\end{align}
where the notations are inherited from Section \ref{sec:preliminary}. Figure~\ref{fig:pipeline} exhibits the entire framework. The following sections detail the definition of components.


\begin{figure*}[ht]
    \centering
    \includegraphics[width=\linewidth]{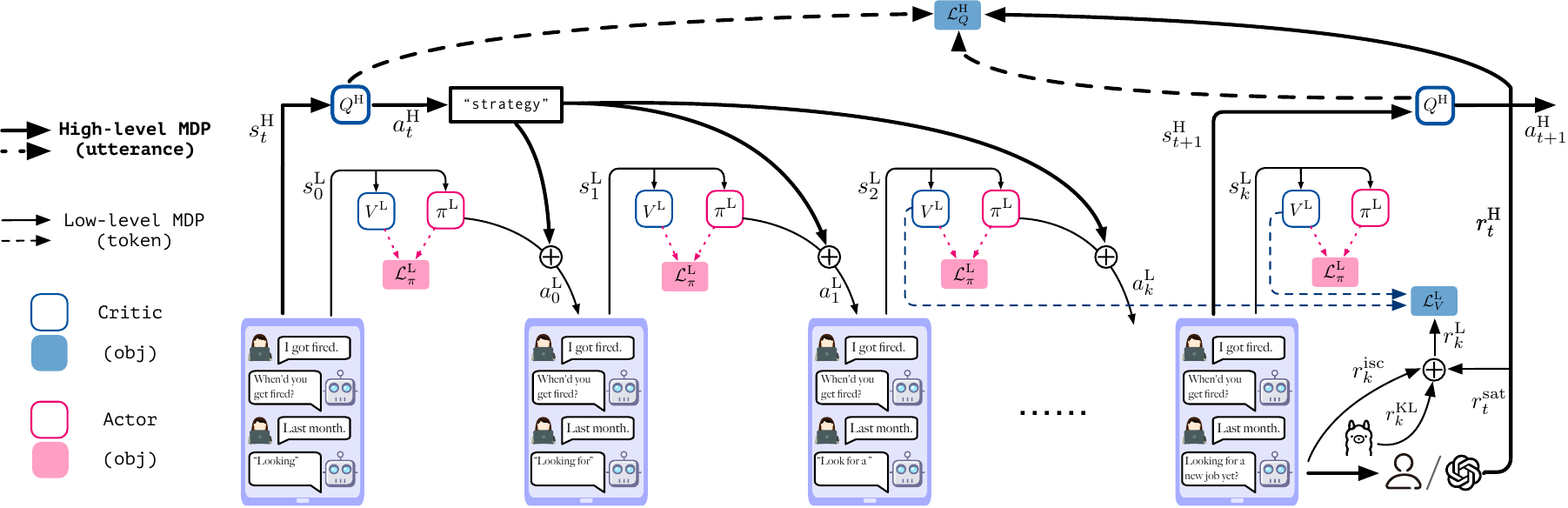}
    \caption{    
    Dual-timescale pipeline of {\ModelName}. The high-level Q-network $Q^\text{H}$ operates at the utterance timescale and selects an optimal $strategy$ at the $t$-th turn. The low-level actor-critic operates at the token timescale and generates the $k$-th token ($a^\text{L}_k$), conditioned on the high-level action $a^\text{H}_t$. The high-level reward ($r^\text{H}_t$) consists solely of the satisfaction score ($r^{\text{sat}}_t$), whereas the low-level reward ($r^\text{L}_k$) additionally incorporates the KL penalty ($r^{\text{KL}}_k$) and intrinsic self-consistency ($r^{\text{isc}}_k$). We optimize $Q^\text{H}$ with the DQN objective ($\mathcal{L}_{Q}^{\text{H}}$), while jointly optimizing $\pi^\text{L}$ and $V^\text{L}$ with PPO objectives $\mathcal{L}_{\pi}^{\text{L}}$ and $\mathcal{L}_{V}^{\text{L}}$, respectively. Solid and dashed lines denote data and gradient flows, respectively.
    }
    \label{fig:pipeline}
\end{figure*}

\subsection{The High-Level MDP} 

\paragraph{Training.} Instead of augmenting the value head (as is done in conventional RLHF), we implement this strategic value function in a purely generative manner. Similar to straQ* \citep{wang-etal-2025-convert}, we calculate the state-action value from the averaged logits of the action tokens, by the conditional forward pass of LLM (with parameters $\phi$):
\begin{align}
    Q_{\phi}^{\text{H}}(s,a) &\leftarrow \text{LLM}_{\phi}(\mathcal{I}^{\text{H}}(s) \oplus a) \label{eq:high_level_Q}
\end{align}
where $\mathcal{I}(s)$ is an instruction template with the placeholder of $s$, and $\oplus$ represents the textual concatenation. We use DQN (Eq. \ref{eq:dqn}) to train $Q_{\phi}^{\text{H}}$.






\paragraph{Prompt.} We provide a simplified version of the high-level instruction $\mathcal{I}^{\text{H}}(s)$ below: 

\tcbset{
  colframe=black!75!white,
  colback=gray!5!white,
  boxrule=0.5pt,
  arc=2mm,
  left=1mm, right=1mm, top=1mm, bottom=1mm,
  fonttitle=\bfseries,
  before skip=5pt, after skip=5pt
}

\begin{tcolorbox}[title=The high-level prompt $\mathcal{I}^{\text{H}}$]
\textbf{Description:} \{$desc$\} 
\ \textbf{\  User's emotion:} \{$emo$\} \\
\textbf{History:} \{$hist$\} \   
\textbf{\  Query:} \{$query$\} \\[2pt]
\textbf{Please select the best strategy:}
\begin{itemize}[
    leftmargin=15pt,  
    labelwidth=10pt,     
    labelsep=5pt,         
    itemindent=0pt,
    align=left,
    itemsep=1pt, topsep=2pt, parsep=0pt
]
  \item[(1)] \{$stra$ 1\} (2) \{$stra$ 2\} $\cdots$ (K) \{$stra$ K\}
\end{itemize}
\end{tcolorbox}
\noindent with the full version in Appendix \ref{appendix:prompt}. Note that we formulate $\mathcal{I}^{\text{H}}$ as a multiple-choice question (MCQ), instead of a plain question, forcing the LLM to choose one of the options without deliberate reasoning. Accordingly, the action set becomes the set of possible strategy indices $a \in \mathcal{A} := \{1, 2, \cdots, K \}$ where $K$ is the total number of strategies.

\paragraph{Inference.} Instead of decoding the next token, the finetuned LLM produces logits of available strategies, and the optimal strategy can be determined from the maximum logit
\begin{equation}
   a^{\text{H}} \leftarrow \arg \max \text{LLM}_{\phi}(\mathcal{I}^{\text{H}}(s) \oplus a), a \in \mathcal{A} \label{eq:high_level_Q_inference}
\end{equation}

\subsection{The Low-Level MDP} 

\paragraph{Training.} We implement the low-level actor $\pi^{\text{L}}$ and the low-level critic $V^{\text{L}}$ on another backbone LLM\footnote{For the critic, an extra value head is augmented.}, with the trainable parameters $\theta$. Similar to RLHF \cite{Ouyang2022ChatGPT}, the token-level PPO is employed to jointly train $\pi^{\text{L}}$ and $V^{\text{L}}$, with objectives shown in Eq. \ref{eq:rlhf_actor} and Eq. \ref{eq:rlhf_critic}.


 

\paragraph{Inference and prompt.} The low-level action $a^{\text{L}}$ is guided by the high-level action $a^{\text{H}}$ (\textit{i.e.}, $stra$) based on in-context learning (ICL):
\begin{align}
    a^{\text{L}} \leftarrow \text{LLM}_{\theta}(\mathcal{I}^{\text{L}}(s^{\text{L}}, a^{\text{H}})) := \pi^{\text{L}}_{a^{\text{H}}}(s^{\text{L}}) \label{eq:high_guides_low_actions}
\end{align}

\noindent Below is a short version of the low-level instruction template $\mathcal{I}^{\text{L}}$ (the full version is in Appendix \ref{appendix:prompt}):


\tcbset{
  colframe=black!75!white,
  colback=gray!5!white,
  boxrule=0.5pt,
  arc=2mm,
  left=1mm, right=1mm, top=1mm, bottom=1mm,
  fonttitle=\bfseries,
  before skip=5pt, after skip=5pt
}

\begin{tcolorbox}[title=The low-level prompt $\mathcal{I}^{\text{L}}$]
\textbf{Description:} \{$desc$\} \textbf{\  User's emotion:} \{$emo$\} \\
\textbf{History:} \{$history$\} \textbf{\  Query:} \{$query$\} \\  
\textbf{Respond based on a reply strategy of \{$stra$\}: }

\end{tcolorbox}

\subsection{The Reward Mechanisms} 

To provide multi-timescale feedback, we define our high-level reward ($r^{\text{H}}$) as the satisfaction score ($r^{\text{sat}}$) at the real time $t$, solely coming from the environment; while the low-level reward ($r^{\text{L}}$) has another two token-level ($k$) components
\begin{align}
    r^{\text{H}}_t = r^{\text{sat}}_t, \quad r^{\text{L}}_k = r^{\text{sat}}_t + \beta_1 r^{\text{KL}}_k + \beta_2 r^{\text{isc}}_k \label{eq:reward_low}
\end{align}
with definitions of reward components below.

\paragraph{The satisfaction score ($r^{\text{sat}}$).} Choice of rewards may be crucial, especially when an offline dataset constrains the sampling. In this paper, we leverage the LLM-as-the-Judge approach to automatically generate online scores. A strong-basis LLM, such as GPT-4o, is employed as the oracle (or the teacher) to produce the satisfaction score $r^{\text{sat}}$ from 0 to 5. In Appendix \ref{appendix:prompt}, we provide the detailed evaluation prompt; in Appendix \ref{appendix:human_gpt4o_consistency}, we conduct statistical experiments to illustrate the consistency between oracle and human scores.




\paragraph{The KL reward ($r^{\text{KL}}$).} The penalty of Kullback–Leibler (K-L) divergence prevents $\pi^{\text{L}}$ from deviating from the reference policy too much, which has the form of  $r^{\text{KL}}=-\KL(\pi^{\text{L}} \Vert \pi^{\text{ref}})$. 

\paragraph{Intrinsic Self-Consistency ($r^{\text{isc}}$).} Here we define the reward of intrinsic self-consistency as the inner-agent alignment between the low-level context (state plus strategy) and response:
\begin{align}
    r^{\text{isc}} = \log {\text{LLM}_{\theta}(s^{\text{L}} \oplus a^{\text{H}} \oplus a^{\text{L}})} \label{eq:reward_im}
\end{align}
where the right-hand side can also be observed as the cross-entropy loss of \text{LLM} with all textual inputs, similar to the pretrained loss in \cite{Ouyang2022ChatGPT}. By augmenting $r^{\text{isc}}$, we encourage the low-level policy to be more proficient with the strategy-steering response style, in contrast with the policy update driven solely by human feedback.

\subsection{The Training Algorithm}

Algorithm~\ref{alg:straChat_hrl} summarizes our training procedure. For each session, we initialize the hierarchical states from a sampled query and perform multi-turn roll-outs. At each turn, the high-level policy selects a strategy, then the low-level policy generates the response token by token and collects the corresponding satisfaction reward. The resulting transitions are stored in separate high- and low-level buffers. We then update $Q^\text{H}$ from the high-level buffer using DQN, while computing the token-level KL and intrinsic self-consistency rewards and jointly optimizing $\pi^\text{L}$ and $V^\text{L}$ with PPO.




\begin{algorithm}[tb]
\caption{HRL-{\ModelName}}
\label{alg:straChat_hrl}
\textbf{Input}: $Q^{\text{H}}$, $V^{\text{L}}$, $\pi^{\text{L}}$ from pretrained checkpoints; batch sizes $B^{\text{H}}$ and $B^{\text{L}}$
\begin{algorithmic}[1] 
\STATE Initialize buffers $\mathcal{B}^{\text{H}}, \mathcal{B}^{\text{L}} = \{ \}$ .
\WHILE{not converges}
    \STATE Initialize $s^{\text{H}}, s^{\text{L}}$ with  queries from dataset
    \STATE \textbackslash \textbackslash Roll-out
    \FOR{t = $1, \dots, T$}
        \STATE Inference $a^{\text{H}}_t$ by  $\arg \max$ (Eq. \ref{eq:high_level_Q_inference})
        \FOR{k = $1, \dots, K$}
            \STATE Inference $a^{\text{L}}_t \sim \pi^{\text{L}}(s^{\text{L}}_t)$ (Eq. \ref{eq:high_guides_low_actions})
            \STATE Increment $s^{\text{L}}_t$ to $s^{\text{L}}_{t+1}$
            \STATE break if {$a^{\text{L}}_t = <EOS>$}

        \ENDFOR
        \STATE collect $r_t^{\text{sat}}$ from the oracle LLM
        \STATE Append buffer $\mathcal{B}^{\text{L}} \cup \{(s^{\text{L}}_k,a^{\text{L}}_k,r_t^{\text{sat}})\}_{k=1}^{K}$ 
        \STATE Append buffer $ \mathcal{B}^{\text{H}} \cup (s^{\text{H}}_t, a^{\text{H}}_t, r_t^{\text{sat}})$

    \ENDFOR

    \STATE \textbackslash \textbackslash High-level Training
    \STATE Sample $(s^{\text{H}}_t,a^{\text{H}}_t,r_t^{\text{sat}},s^{\text{H}}_{t+1},a^{\text{H}}_{t+1})$ from $\mathcal{B}^{\text{H}}$
    \STATE update $Q^{\text{H}}$ by DQN (Eq. \ref{eq:dqn})

    \STATE \textbackslash \textbackslash Low-level Training
    \FOR{k = $1, \dots, K-1$}
         \STATE Get $(s^{\text{L}}_k,a^{\text{L}}_k,r_t^{\text{sat}},s^{\text{L}}_{k+1},a^{\text{L}}_{k+1})$ from $\mathcal{B}^{\text{L}}$
         \STATE Calculate $r_k^{\text{KL}}$ , $r_k^{\text{isc}}$, $r_k^{\text{L}}$ by Eq. \ref{eq:reward_low}
        \STATE Update $\pi^{\text{L}}$ and $V^{\text{L}}$ jointly by PPO (Eq. \ref{eq:rlhf_actor})
    \ENDFOR
\ENDWHILE
\end{algorithmic}
\end{algorithm}

\subsection{Theoretical Derivation}

We start from two assumptions proposed by the \textit{options} framework \cite{10.1016/S0004-3702(99)00052-1}, which are the foundation of the subsequent main theorem.

\begin{assumption}
\label{assumption:everywhere}
The strategies, as special types of \textit{options}, are everywhere:
\begin{align*}
    \forall s \in \mathcal{S}, \forall a^{\text{H}} \in \mathcal{A^{\text{H}}}: \exists s \in \mathcal{I}(s|a^{\text{H}})
\end{align*}
where $\mathcal{I}$ is the initial state set.
\end{assumption}

\begin{assumption}
\label{assumption:markovian}
The strategies are Markovian:
\begin{equation*}
    \text{P}(s_t^{\text{H}}, a_t^{\text{H}} | s_{0:t-1}^{\text{H}}, a_{0:t-1}^{\text{H}}) = \text{P}(s_t^{\text{H}}, a_t^{\text{H}} | s_{t-1}^{\text{H}}, a_{t-1}^{\text{H}})
\end{equation*}
\end{assumption}

Considering the features of the dialogue system, it is obvious that both assumptions are reasonable: the agent can choose any strategy or any response, no matter what the current state is; the prior history and query strategy are also prerequisites for determining the next strategy. Within a hierarchical framework, we propose the main theorem:

\begin{theorem} 
\label{theorem:convergence_strategy}
With Assumptions \ref{assumption:everywhere} and \ref{assumption:markovian} holding, the value function (Eq. \ref{eq:high_level_Q}) converges as long as the high-level policy (Eq. \ref{eq:high_level_Q_inference}) is deterministic, and the high- and low-level critics are solved interleaved.
\end{theorem}
Due to page limits, we leave the detail proof in Appendix \ref{appendix:theory}. Theorem \ref{theorem:convergence_strategy} suggests that our determined strategy can converge to the ground truth.

\section{Experiment}

In this section, we investigate two research questions (RQs): \textbf{RQ1}, \textit{whether our framework can adapt across different conversational domains}; and \textbf{RQ2}, \textit{whether our method can generalize to settings without strategy annotations}. For \textbf{RQ1}, we evaluate our framework on two different domains in Sections \ref{sec:dailylife_result} and \ref{sec:esc_result}, respectively. For \textbf{RQ2}, we further distinguish between in-domain (ID) and out-of-domain (OOD) evaluations according to the availability of strategy annotations.



\subsection{Implementation} We use LLaMA3.2-1B-Instruct and LLaMA3.1-8B-Instruct \cite{llama3modelcard} as high-level and low-level backbones, respectively. The framework is trained on OpenRLHF \cite{hu2024openrlhf} using 16 A100 GPUs. We use learning rates of 5e-6, 9e-7, and 9e-4 for training of $Q^{\text{H}}$, $\pi^{\text{L}}$, and $V^{\text{L}}$, respectively. The batch size to 64 and the max decoding length is set to 128.  We use GPT-4o as the reward evaluator to provide satisfaction scores.





\subsection{Datasets} 


\paragraph{Daily-life conversations.} DailyDialog \cite{li-etal-2017-dailydialog} is a manually constructed multi-turn dialogue dataset containing approximately 13,000 conversations on everyday topics. Each utterance is annotated with a conversational strategy\footnote{Denoted as `act' in the original dataset.}, including \textit{Question}, \textit{Inform}, \textit{Directive}, and \textit{Commissive}.

\paragraph{Emotional support conversations (ESC).} We consider two widely used datasets in this domain: ESConv \cite{liu2021ESconv}, which is annotated with $8$ supporting strategies, and EmpatheticDialogues \cite{rashkin-etal-2019-towards}, which contains \textbf{no} strategy annotations. Accordingly, we use ESConv for the ID setting, where the model is trained on its training set and evaluated on its test set. We then use EmpatheticDialogues for the OOD setting to evaluate the model trained on ESConv.

The ID evaluations on DailyDialog and ESConv address \textbf{RQ1} by covering two distinct conversational domains. In contrast, the OOD evaluation addresses \textbf{RQ2} by examining the transferability of {\ModelName} to a different dialogue dataset without strategy annotations. Appendix \ref{appendix:dataset} provides more detailed descriptions and statistics of these datasets.









\subsection{Evaluation}

\paragraph{Classification metrics.} We employ classification metrics of accuracy (Acc) and Macro-F1 (MaF1). We also refer to the evaluation methods proposed by \citet{kang-etal-2024-large}, which propose the $bias$ of strategies based on the Bradley-Terry model~\citep{bradley1952btmodel}. A smaller $bias$ indicates more balanced strategic determination.


\paragraph{Automatic metrics on response.} For generative tasks, we utilize similarity-based metrics like Bleu-2 (B-2) and Rouge-L (R-L); as well as Distinct-2 (D-2), which indicates response diversity. 


\paragraph{Human scoring.} Similar with \citet{kang-etal-2024-large}, we annotate with dimensions of \textit{Acceptance}, \textit{Effectiveness}, \textit{Sensitivity}, \textit{Fluency}, and \textit{Emotion}, and the ultimate purpose, seeker's \textit{Satisfaction}. The detailed annotation settings and criteria are provided in Appendix \ref{appendix:huam_score_principle}.

\subsection{Baselines}


We consider several categories of baselines:

\paragraph{Training-free methods.} These methods induce strategic or structured reasoning without updating the backbone LLM parameters, including \textbf{Self-Refine} \cite{Madaan2023SelfRefine}, Emotional Chain of Thoughts (\textbf{ECoT}) \cite{li2024enhancingemotionalgenerationcapability}, Skeleton of Thoughts (\textbf{SoT}) \cite{ning2024skeletonofthought}, Tree of Thoughts (\textbf{ToT}) \cite{yao2023tree}, Plan-and-Solve (\textbf{PS}) \cite{wang2023planandsolvepromptingimprovingzeroshot}, Finite State Machine (\textbf{FSM}) \cite{wangFSMFiniteState2024}, as well as \textbf{Direct-Refine}, a single-turn variant of Self-Refine, and \textbf{Direct}, which directly inference with the backbone.

\paragraph{Supervised learning methods.} These methods directly leverage the `golden' responses to finetune the LLM, including supervised finetuning (\textbf{SFT}), as well as the supervised variants of ECoT (\textbf{SFT + ECoT}) and FSM (\textbf{EmoFSM})\footnote{For both methods, we first inference the model on the training set according to the corresponding prompting mechanism, then finetune the LLM with the inferenced response.} \cite{10.1007/978-981-92-1926-1_10}. 

\paragraph{Reinforcement learning methods.} We compare against different RL paradigms to evaluate the effectiveness of {\ModelName}; these baselines also provide natural references for ablating different levels of our hierarchical framework:

\noindent (1) \textbf{PPO}, a representative \textit{token-level policy optimization} method widely adopted in RLHF \citep{Ouyang2022ChatGPT}.

\noindent (2) \textit{Utterance-level planners}, including \textbf{straQ*} \citep{wang-etal-2025-convert}, which trains an LLM with DQN to produce semantic strategy-level decisions, and \textbf{DAT} \cite{liDialogueActionTokens2024}, which represents each dialogue utterance as a continuous vectorized action and optimizes multi-turn decision-making. Both methods learn lightweight high-level planners that subsequently steer LLM generation. 

\noindent (3) \textbf{ArCHer} \cite{pmlr-v235-zhou24t}, an HRL-based method that combines an \textit{utterance-level critic} with a \textit{token-level actor}, providing a hierarchical baseline with both temporal levels.

Detailed introductions and implementations of the baselines can be found in Appendix \ref{appendix:baseline}.

\subsection{Results on Daily-Life Conversations}
\label{sec:dailylife_result}



\paragraph{Return optimization.} We first examine the reward and return estimates during training. We run the experiment on DailyDialog five times with different random seeds and report the mean and standard deviation of $Q^H$ and $r^L$ in Table~\ref{tab:reward_and_value}. {\ModelName} achieves higher average rewards and returns than baselines, with $p$-values below 0.01, indicating effective return optimization. Specifically, the significance test $H_0: \text{Metric}_X \geq \text{Metric}_{\ModelName}$ yields a $p$-value below 0.01 for both baselines $X$ and both metrics. We also observe stable convergence, with several training curves provided in Appendix \ref{appendix:train_curve}.




\begin{table}[t!]
\centering
\small
\begin{tabular}{ccc}
\toprule
\textbf{Method} & $<Q^{\text{H}}>$  & $<r^{\text{L}}>$  \\
\hline
Direct  & $486.2 \pm 1.22$ & $3.21 \pm 0.32$ \\
Raw Dataset   & $551.4 \pm 1.33$  & $3.53 \pm 0.35$  \\
\textbf{\ModelName} (ours) & $\textbf{616.5} \pm 1.35$ & $\textbf{4.03} \pm 0.34$ \\
\bottomrule
\end{tabular}
\caption{Averaged values and rewards on DailyDialog.}
\label{tab:reward_and_value}
\end{table}

\paragraph{Strategy determination.} As shown in Table \ref{tab:ID_result_DailyDialog}, \ModelName{} achieves the highest Acc despite receiving no reward signal from the ground-truth strategy labels during training. This suggests that the learned high-level policy aligns well with the human-annotated strategy space under the automatic reward signal. Although \ModelName{} does not achieve the best $bias$ score, its competitive performance on this metric, together with the highest Acc, results in the best MaF1 of 58.91.


\paragraph{Response quality.} Table \ref{tab:ID_result_DailyDialog} further shows that \ModelName{} achieves the highest similarity to the ground-truth responses, as measured by B-2 and R-L, while also attaining the best diversity in terms of D-2. It consistently outperforms prompting, finetuning and RL-based baselines on DailyDialog. Previously reported results using different backbones are provided  in Appendix \ref{appendix:result_w_diff_backbone}.


\begin{table}[t!]
\renewcommand{\arraystretch}{1.11}
\centering
\resizebox{\columnwidth}{!}{
\begin{tabular}{l cccccc}
    \toprule
    Methods & Acc $\uparrow$ & MaF1 $\uparrow$ & $bias$ $\downarrow$ & B-2 $\uparrow$ & R-L $\uparrow$ & D-2 $\uparrow$ \\
    \midrule
    Direct & 52.60 & 18.03 & 1.66 & 3.35 & 10.33 & 44.74 \\ 
    Direct-Refine & 48.27 &	28.28 &	0.70 &	2.56 &	8.70	& 43.67 \\
    Self-Refine \cite{Madaan2023SelfRefine} & 49.76 &	22.15 &	1.18 &	2.40	& 7.75 & 34.01 \\
    ECoT \cite{li2024enhancingemotionalgenerationcapability} & 38.94 &	29.99 &	0.27 &	1.78 &	6.00 &	55.26 \\
    SoT \cite{ning2024skeletonofthought} & N/A & N/A & N/A & 2.53 & 7.97 & 59.98  \\
    ToT  \cite{yao2023tree} & N/A & N/A & N/A & 2.52 & 8.84 &	43.19  \\
    PS \cite{wang2023planandsolvepromptingimprovingzeroshot} & N/A & N/A & N/A & 2.60 & 7.76 & 39.73  \\
    FSM \cite{wangFSMFiniteState2024} & 46.86 & 21.22 &	1.30	& 2.70 &	9.44 & 38.75 \\
    \midrule
    SFT & 60.19 & 44.82 &	0.82 &6.81 & 18.52 & 43.36 \\
    SFT+ECoT &60.11 & 44.9 &	0.66 &	6.61 &	18.07 &	42.87 \\
    EmoFSM \cite{10.1007/978-981-92-1926-1_10} & 60.03 & 46.02 & 0.55 & 5.85 & 21.77 & 47.43 \\
    \midrule
    PPO & N/A & N/A & N/A & 7.85 & 25.16 & 50.59 \\
    DAT \cite{liDialogueActionTokens2024}  & N/A & N/A & N/A & 3.45 & 11.80 & 0.90 \\
    straQ* \citep{wang-etal-2025-convert} & 54.01 & 50.10 & 0.62 & 4.18 & 13.09 & 59.27 \\
    ArCher \cite{pmlr-v235-zhou24t} & 50.41 & 42.67 & \textbf{0.21} & 5.17 & 14.35 & 55.16 \\
    
    \textbf{{\ModelName}} (ours) & \textbf{63.64} & \textbf{58.91} & 0.63 & \textbf{16.35} & \textbf{35.22} & \textbf{62.67} \\ 
    \bottomrule
\end{tabular}}
\caption{ID results of automatic metrics including Acc, MaF1, $bias$, B-2, R-L and D-2 on DailyDialog. The best result of the methods is \textbf{bolded}. 
}
\label{tab:ID_result_DailyDialog}
\end{table}


\subsection{Results on ESC}
\label{sec:esc_result}


\paragraph{In-domain results.} As shown in Table \ref{tab:ID_result_ESconv}, when moving from daily-life dialogue to emotional support conversations, {\ModelName} achieves the best or second-best performance across all metrics, indicating that it can select appropriate support strategies while maintaining strong response quality. We further report human evaluation results in Appendix \ref{appendix:human_eval_result}. Typical cases are provided in Appendix \ref{appendix:case}.



\paragraph{Transfer to OOD.} In this setting, we exclude finetuning baselines to better reflect the knowledge-transfer scenario. As shown in Table~\ref{tab:OOD_result_empatheticdialogues}, \ModelName{} achieves the best overall performance on EmpatheticDialogues in this zero-shot setting, demonstrating strong out-of-domain generalization. This result suggests that the empathy-related behaviors learned from ESConv can be effectively transferred to a dataset without strategy annotations.


\begin{table}[t!]
\renewcommand{\arraystretch}{1.11}
\centering
\resizebox{0.98\columnwidth}{!}{
\begin{tabular}{l cccccc}
    \toprule
    Methods & Acc $\uparrow$ & MaF1 $\uparrow$ & $bias$ $\downarrow$ & B-2 $\uparrow$ & R-L $\uparrow$ & D-2 $\uparrow$ \\
    \midrule
    Direct & 11.80 &	10.26 &	1.61 &	3.47 & 10.64 & 33.45 \\ 
    Direct-Refine & 17.08 &	11.07 &	1.27 &	3.10 &	6.13 &	14.22 \\
    Self-Refine \cite{Madaan2023SelfRefine} & 17.58 &	13.61 &	1.92 &	3.34 &	9.71 &	14.61 \\
    ECoT \cite{li2024enhancingemotionalgenerationcapability} & 15.32 &	10.38 &	1.69 &	3.16 &	10.50 &	30.38 \\
    SoT \cite{ning2024skeletonofthought} & N/A & N/A & N/A &	3.07 &	8.76 &	26.15 \\
    ToT  \cite{yao2023tree} & N/A & N/A & N/A & 2.65 &	9.81 &	14.2  \\
    PS \cite{wang2023planandsolvepromptingimprovingzeroshot} & N/A & N/A & N/A & 2.81 &	8.27 &	19.12  \\
    FSM \cite{wangFSMFiniteState2024} & 17.37 &	11.15 &	0.81 &	4.12 &	11.83 &	35.43 \\
    \midrule
    SFT & 32.43 & 21.29 & 1.28 & \bf 6.97 & \bf 16.59 & 50.45 \\
    SFT+ECoT & 30.80 &	17.70 &	1.35 &	6.51 &	15.00 &	34.96 \\
    EmoFSM \cite{10.1007/978-981-92-1926-1_10} & 28.00 &	23.70 &	\textbf{0.41} &	5.88 &	15.30 &	51.48 \\
    \midrule
    PPO & N/A & N/A & N/A & 6.76 & 15.45 & 50.93 \\ 
    DAT \cite{liDialogueActionTokens2024} & N/A & N/A & N/A & 3.26 & 11.24 & 35.82 \\ 
    straQ* \citep{wang-etal-2025-convert}  & \underline{37.69} & \underline{34.57} & 0.59 & 3.59 & 11.74 & 44.14 \\
    ArCher \cite{pmlr-v235-zhou24t} & 24.50 & 19.60 & 0.50 & 5.30 & 13.10 & \textbf{54.80} \\
    \textbf{{\ModelName}} (ours) & \textbf{39.26} & \textbf{36.85} & \underline{0.48} & \underline{6.93} & \underline{16.28} & \underline{52.42} \\ 
    \bottomrule
\end{tabular}}
\caption{ID results of automatic metrics including Acc, MaF1, $bias$, B-2, R-L and D-2 on ESConv. The best result is \textbf{bolded} and the second best is \underline{underlined}.
}
\label{tab:ID_result_ESconv}
\end{table}

\begin{table}[htbp!]
\centering
\small
\begin{tabular}{l cccc}
    \toprule
    \textbf{Methods} & B-2 $\uparrow$ & R-L $\uparrow$ & D-2 $\uparrow$ \\
    \midrule
    Direct & 3.09 & 9.91 & 25.23  \\
    ECoT\cite{li2024enhancingemotionalgenerationcapability} & 2.91 & 9.79 & 32.65 \\
    SoT \cite{ning2024skeletonofthought}  & 1.79 & 5.66 & \textbf{48.59} \\
    ToT  \cite{yao2023tree}  & 2.31 & 9.05 & 29.09 \\
    PS \cite{wang2023planandsolvepromptingimprovingzeroshot} & 2.69 & 6.93 & 24.02 \\
    FSM \cite{wangFSMFiniteState2024} & 3.33 & 10.80 & 33.37  \\
    \midrule
    PPO & \underline{3.91} & \underline{11.16} & 41.59 \\
    DAT \cite{liDialogueActionTokens2024} & 3.38 & 10.90 & 35.05 \\
    straQ* \citep{wang-etal-2025-convert} & 3.61 & 9.15 & 39.60 \\
    ArCHer \cite{pmlr-v235-zhou24t} & 3.84 & 10.75 & 45.19 \\
    \textbf{{\ModelName}} (ours) & \bf 4.49 & \bf 12.93  & \underline{46.53}  \\
    \bottomrule
\end{tabular}
\caption{OOD results of B-2, R-L and D-2 on EmpatheticDialogues. The best result of the methods is \textbf{bolded} and the second best is \underline{underlined}. 
}
\label{tab:OOD_result_empatheticdialogues}
\end{table}

\begin{table}[t!]
\centering
\small
\resizebox{0.98\columnwidth}{!}{
\begin{tabular}{l cccccc}
    \toprule
    \textbf{Method} & Acc $\uparrow$ & MaF1 $\uparrow$ & $bias$ $\downarrow$ & B-2 $\uparrow$ & R-L $\uparrow$ & D-2 $\uparrow$ \\
    \midrule
    w/o high-level & N/A & N/A & N/A & 7.85 & 25.16 & 50.59  \\
    w/o low-level & 53.39 & 49.02 & \textbf{0.53} & 4.17 & 13.15 & 59.44  \\
    w/o isc & 60.04 & 53.01 & 0.68 & 16.21 & 35.06 & 60.95 \\
  \textbf{{\ModelName}} (ours)  & \bf 63.64 & \bf 58.91 & 0.63 & \bf 16.35 & \bf 35.22 & \textbf{62.67} \\
    \bottomrule
\end{tabular}}
\caption{Ablation study of {\ModelName} on DailyDialog. 
}
\label{tab:ablation}
\end{table}


\paragraph{Ablation Study.} We consider three ablations: (1) \textit{w/o high-level}, which removes the high-level Q-network and reduces {\ModelName} to token-level PPO; (2) \textit{w/o low-level}, which removes low-level policy optimization and reduces the framework to a DQN-based strategy planner (\textit{i.e.}, straQ* \cite{wang-etal-2025-convert}) guiding a frozen LLM; and (3) \textit{w/o isc}, which removes intrinsic self-consistency from the low-level reward. The first two variants are identical to the PPO and straQ* baselines reported in the main results, respectively. We include them there because they are strong and representative comparison methods, while they also serve as natural ablations of our hierarchical architecture. As shown in Table~\ref{tab:ablation}, removing either level leads to a substantial performance drop, supporting the effectiveness of the hierarchical design. Removing intrinsic self-consistency also yields a consistent decline in metrics, further validating its contribution. Appendix~\ref{appendix:dataset} presents additional sensitivity analyses on $\gamma$ and the reward weights ($\beta_1$ and $\beta_2$).


\subsection{Discussion}

\paragraph{Method variance across domains.} Tables \ref{tab:ID_result_DailyDialog} and \ref{tab:ID_result_ESconv} reveal substantial performance variation across domains for several baselines. For example, SFT and ArCher perform comparably to, or even outperform, {\ModelName} on ESConv, but perform substantially worse on DailyDialog. One possible explanation is that ESConv focuses on emotional support with relatively limited topics and clear conversational objectives, leading to more structured behavior patterns (Exploration->Comforting->Action, as in \cite{liu2021ESconv}) that are easier for supervised methods to imitate. In contrast, DailyDialog contains more open-ended topics, spontaneous emotional exchanges, and broader behavioral distributions, making pure imitation more challenging. Despite these cross-domain variations, {\ModelName} maintains robust performance across all three datasets, indicating strong generalization and adaptation.

\paragraph{Intrinsic self-consistency for strategic chatting.} Our framework uses high-level strategies to guide the low-level policy through in-context conditioning. However, intensive post-training for return optimization may weaken LLM's adherence to these strategy instructions. We therefore introduce intrinsic self-consistency as an auxiliary reward to encourage strategy-aligned behavior during multi-turn training. Its effectiveness is supported by the ablation results in Table \ref{tab:ablation}. A representative case in Table~\ref{tab:im} further shows that without intrinsic self-consistency, the model fails to follow the high-level strategy (\textit{Question}). Additional examples are in Appendix~\ref{appendix:result}.



\begin{table}[h!]
    \centering
    \footnotesize
    \resizebox{0.98\columnwidth}{!}{
    \begin{tabular}{p{1.2cm}|p{6.7cm}}
        \toprule
        \textbf{User} & It's a portable TV. It's a popular thing now. \\
        \midrule
        \textbf{Assistant} & Oh, that's new to me. So what's on everyday? \\
        \midrule
        \textbf{User} & News about current affairs, documentaries, music, movies, noncommercial ads and so on. \\
        \midrule
        \textbf{w/o isc} & (\textit{Question}) \textcolor{red}{I bet there’s something fun to watch on it.} \\
        \midrule
        \textbf{w/ isc} & (\textit{Question}) Got anything exciting on there? \\
        \bottomrule
    \end{tabular}}
    \caption{Comparison of responses with and without intrinsic self-consistency (isc). The response of w/o isc shows poor strategy-response consistency (in \textcolor{red}{red}).} 
    \label{tab:im}
\end{table}

\paragraph{Lookahead strategy determination.} To better understand the learned strategy dynamics, we further analyze strategy transitions, where each grid cell $(i, j)$ denotes the transition frequency from the strategy in the $i$-th row to that in the $j$-th column.

Figure~\ref{fig:strategy_matrix} (Left) shows normalized strategy transitions from the user to {\ModelName} within the same turn. \textit{“Question $\rightarrow$ Inform”} and \textit{“Directive $\rightarrow$ Commissive”} are the two most frequent transitions, consistent with common conversational patterns in which the agent responds to a question or commits to providing assistance.

Figure~\ref{fig:strategy_matrix} (Right) shows transitions between {\ModelName}'s strategies across adjacent turns. The upper-right triangle reveals a statistically prominent pattern, \textit{Inform $\rightarrow$ Question $\rightarrow$ Directive $\rightarrow$ Commissive}, which is also consistent with the Helping Skills Theory discussed in \citet{liu2021ESconv}. This observation suggests that {\ModelName} can anticipate subsequent conversational needs through higher-level strategy planning. 

\section{Related Work}

\subsection{Hierarchical Reinforcement Learning}

Hierarchical Reinforcement Learning (HRL) decomposes decision-making into hierarchical MDPs, with representative paradigms including \textit{goal-conditioned} frameworks \cite{Kulkarni2016h-DQN,Vezhnevets2017FuN,Levy2019HAC} and \textit{options}-based frameworks with discrete subskills \cite{10.5555/3298483.3298491,chunduru2020attention,NEURIPS2019_4f284803}. Our framework follows the latter paradigm, with explicit semantic connections between high- and low-level actions.

\subsection{Utterance-level RL for Dialogue Systems}


Utterance-level RL complements token-level RL \cite{Ouyang2022ChatGPT} by introducing coarser action abstractions to alleviate reward sparsity and reduce exploration complexity. For example, DAT \cite{li2024dialogue} models utterance-level actions in a continuous space, while ArCHer \cite{pmlr-v235-zhou24t} adopts a two-level HRL framework with a task-oriented high-level critic. Our method also employs a two-level HRL, but introduces explicit, high-level strategy abstraction to guide token-level generation, providing a semantic interface between high-level planning and token-level generation.




\begin{figure}[t]
    \centering
    \includegraphics[width=0.45\linewidth]{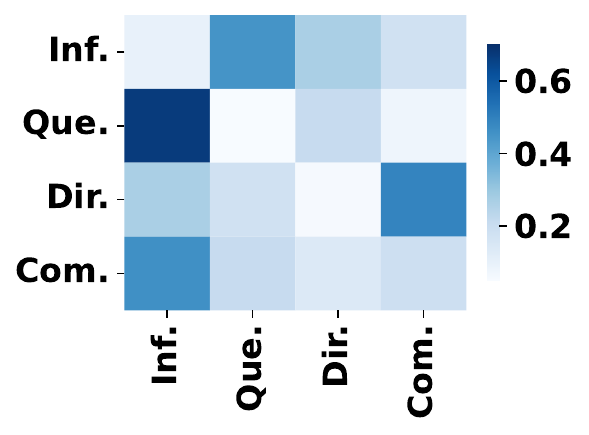}
    \hspace{0.1in}
    \includegraphics[width=0.45\linewidth]{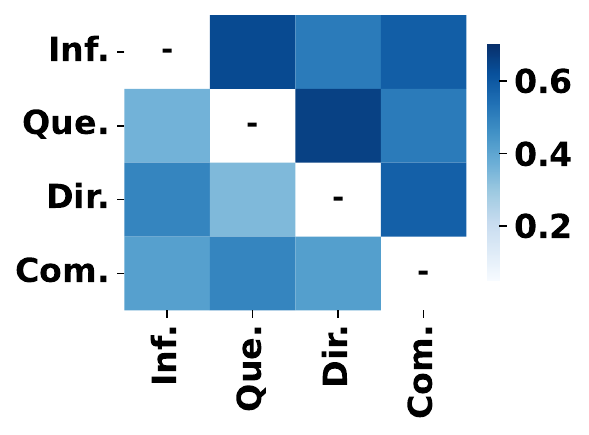}
    \caption{Occurrences of strategy (in abbrev.) transitions. \textit{Left:} Strategy transitions from user (row) to {\ModelName} (column). \textit{Right:} Strategies transitions of {\ModelName} between adjacent turns (row $\rightarrow$ column).} 
    \label{fig:strategy_matrix}
\end{figure}

\section{Conclusion}


In this paper, we propose {\ModelName}, an HRL-based dialogue agent with temporal and strategic abstractions. {\ModelName} employs a high-level Q-network to select strategy-level actions, which guide the low-level policy in generating detailed responses. To improve sampling efficiency and training stability, we optimize the high-level Q-network with DQN and the low-level actor-critic with PPO. To further strengthen the alignment between the two levels, we incorporate a KL penalty and an intrinsic self-consistency reward into the low-level optimization. Experiments on daily-life and emotional support conversations, together with the OOD evaluation, demonstrate that {\ModelName} effectively combines explicit strategic planning with token-level policy optimization across diverse conversational settings.

\clearpage
\newpage





\section*{Limitations}

Our goal is not to claim that the proposed method is universally optimal, but to examine whether explicit strategy abstraction benefits conversational RL in domains with well-defined strategies. {\ModelName} relies on a predefined strategy set, which constrains the learned policy and limits its applicability to domains with established or reliably annotated strategies. Moreover, the current strategy-then-token design selects a single strategy per turn, whereas real utterances may involve multiple strategies. Future work may explore compositional, dynamically generated, or automatically discovered strategies to improve flexibility and generalization.

\section*{Ethical Considerations}

Strategy-aware agents may introduce risks because high-level decisions can shape how models influence users. An improperly designed strategy space could encourage aggressive, manipulative, or otherwise undesirable behaviors. In {\ModelName}, the high-level policy is restricted to a predefined strategy set, making available behaviors explicit and easier to control or filter. Practical deployment should nevertheless carefully audit the strategy taxonomy, reward signals, and behavioral constraints.

\bibliography{main}

\newpage
\appendix

\section{Implementation Details}
\label{appendix:implementation}

\subsection{Detailed Prompts}
\label{appendix:prompt}

\paragraph{Prompt of the high-level agent.}

Below is the high-level instruction template. By querying the LLM in a multi-choice question style, LLM is encouraged to decode the next token as the index of one of the candidate strategies. User emotion, conversation description, and history are also included.

\begin{promptbox}{The high-level prompt}
You are given a multi-turn dialogue between a user and an assistant. The user's basic situation is as follows:\\
Emotion: \{$emo$\}\\
Description: \{$desc$\}\\

Below is the dialogue history between the user and the assistant:\\
\{$history$\}\\

The user's current query is:\\
\{$query$\}\\

Based on the above context, please select the most appropriate response strategy from the following options:\\
strategy \#(1) \{$a_1$\}\\
...\\
strategy \#(k) \{$a_k$\}\\
Please provide your selection in the format of (1) through (k). Your selection is:
\end{promptbox}

\paragraph{Prompt of the low-level agent.}

The low-level instruction template steers the LLM by the strategy from the high-level, and generates an enhanced response.

\begin{promptbox}{The low-level prompt}
You are given a multi-turn dialogue between a user and an assistant. The user's basic situation is as follows:\\
Emotion: \{$emo$\}\\
Description: \{$desc$\}\\

Below is the dialogue history between the user and the assistant:\\
\{$history$\}\\

The user's current query is:\\
\{$query$\}\\

The current response strategy is:\\
\{$stra$\}\\

Based on the current response strategy and other information, please act as an assistant and provide the best response. Keep replies brief without additional pronouns or extra elements.
\end{promptbox}


\paragraph{The evaluation prompt of GPT-4o.}

The model responses are evaluated by GPT-4o from the human perspective, across five dimensions: \textit{Acceptance}, \textit{Effectiveness}, \textit{Sensitivity}, \textit{Fluency}, and \textit{Emotion}. Each dimension is scored as either 0 or 1, indicating whether the response is strong in that aspect. The final output includes the total and detailed scores, as well as a brief explanation:

\begin{promptbox}{Evaluation prompt}
You are an expert evaluator simulating human assessment of dialogue responses. 
You are given a multi-turn dialogue between a user and an assistant. The user's basic situation is as follows:\\
    Emotion: \{$emo$\}\\
    Description: \{$desc$\}\\
Below is the dialogue history between the user and the assistant:\\
    \{$history$\}\\

The user's current query is:\\
    \{$query$\}\\

The assistant's response is:\\
    \{$response$\}\\

Please evaluate the assistant's response across five binary-scored dimensions. \\

\textit{1. Acceptance} -- Is the response socially appropriate and non-offensive? \\
\textit{2. Effectiveness} -- Does the response address the user's intent appropriately? \\
\textit{3. Sensitivity} -- Does the response consider emotional or situational context? \\
\textit{4. Fluency} -- Is the response grammatically correct and fluent? \\
\textit{5. Emotion} -- Does the response convey an appropriate emotional tone? \\

Each dimension should be scored as \textbf{0 (unsatisfactory)} or \textbf{1 (satisfactory)}.\\

\textbf{{Output format:}} \\
\textit{Acceptance}\textit{: [0/1]}, 
\textit{Effectiveness}\textit{: [0/1]}\\
\textit{Sensitivity}\textit{: [0/1]},
\textit{Fluency}\textit{: [0/1]}\\
\textit{Emotion}\textit{: [0/1]}, 
\textit{Total score}\textit{: [0--5]} \\
\textit{Explanation}\textit{: [your reasoning here]}
\end{promptbox}

\subsection{More Methodology Details}
\label{appendix:more_method}

\paragraph {Kullback-Leibler (KL) Divergence} is a fundamental concept in information theory and machine learning that measures the difference between two probability distributions. It is a crucial tool for comparing and approximating distributions, with far-reaching implications in various fields. The KL Divergence between two probability distributions $P$ and $Q$ is defined as: s

\begin{equation}
D_{KL}(P||Q) = \sum_{i=1}^N p(x) \log\frac{p(x)}{q(x)}
\end{equation}
where $p(x)$ and $q(x)$ are the probability density functions of $P$ and $Q$ respectively.

\section{Experiment Settings}
\label{appendix:experiment_setting}

\subsection{Datasets}
\label{appendix:dataset}

\paragraph{Detailed introductions.} We introduce our experimental datasets in more details:
\begin{itemize}
    \item The \textit{DailyDialog} dataset \cite{liDailyDialogManuallyLabelled2017} is  a widely used benchmark for daily-life conversation. It contains 13,118 multi-turn dialogues across diverse daily topics (e.g., family, work, hobbies) with manually annotated emotions and dialogue strategies (named `act' in the original dataset). As shown in Table~\ref{tab:strategies}, it is annotated with 7 emotions and 4 strategies: \textit{Inform(Inf.)} , \textit{Question(Que.)} , \textit{Directives(Dir.)}, and \textit{Commissive(Com.)}.
    \item The \textit{ESConv} dataset \cite{liu2021ESconv} is a specialized corpus for emotional support research. It contains 1,000+ multi-turn dialogues where users articulate personal struggles (e.g., workplace stress, interpersonal conflicts, self-esteem issues), and support providers respond with empathy, validation, and actionable guidance. Aligned with Table~\ref{tab:strategies}, it is annotated with 11 emotions and 8 conversational strategies. The strategies are structured across three stages (I–III).
    \item The \textit{EmpatheticDialogues} dataset \cite{rashkin-etal-2019-towards} contains conversations centered on recognizing and responding to emotions (e.g., joy, sadness, frustration). It has annotations of user emotions but no strategies.

\end{itemize}

\begin{table}[ht!]
\centering
\resizebox{0.95\columnwidth}{!}{ 
\begin{tabular}{c|l|c|c}  
\toprule
\textbf{Dataset} & \multicolumn{1}{c|}{\textbf{Strategies}} & \multicolumn{1}{c|}{\textbf{Abbr.}} & \multicolumn{1}{c}{\textbf{Stage}}  \\ 
\toprule
\multirow{4}{*}{\makecell[c]{DailyDialog}} & Inform & \bf Inf.& - \\ 
& Question &\bf Que. & -  \\ 
& Directive &\bf Dir. & - \\ 
& Commissive &\bf Com.& - \\ 
\midrule 
\multirow{8}{*}{\makecell[c]{ESConv}} & Question & \bf Que.& I \\ 
& Restatement or Paraphrasing &\bf Res.\& Par.& I \\ 
& Reflection of Feelings &\bf Ref.& II \\ 
& Self-disclosure &\bf Self-Dis.& II \\ 
& Affirmation and Reassurance &\bf Aff.\& Rea.& III \\
& Providing Suggestions &\bf Pro.& III \\ 
& Information & \bf Inf.& III \\ 
& Others &\bf Others& - \\ 
\bottomrule
\end{tabular}
}
\caption{Strategies, abbreviations, and stages within DailyDialog and ESConv.} 
\label{tab:strategies}
\end{table}


\paragraph{Statistics of datasets.} Table \ref{tab:statistics} summarizes the statistics of ESConv, DailyDialog, and EmpatheticDialogues. It presents key metrics like the number of sessions, utterances, average utterance length, and counts of strategies and emotions for both speakers in each dataset. For instance, DailyDialog has the most sessions (13.1k) and utterances (103.0k in total for both speakers), while ESConv shows detailed strategy annotations (11 for Speaker1) and a moderate number of sessions (1.3k). EmpatheticDialogues, while not annotated with strategies, offers a rich set of 32 emotion categories, making it valuable for studying emotional diversity and text-based evaluation.

\begin{table}[t!] 
  \centering
  \resizebox{\columnwidth}{!}{
    \begin{tabular}{llccc}
    \toprule
    \multicolumn{2}{c}{\textbf{Category}} &\textbf{ESconv} & \textbf{DailyDialog} & \textbf{Empathetic-}  \\
    \midrule
    \multicolumn{2}{l}{\# Sessions} & 1.3K & 13.1k& 2.5K \\
    \multicolumn{2}{l}{\# Utterances} & 38K & 103.0k& 11.0K\\
    \multicolumn{2}{l}{Average \# Utterances} & 28.9  & 7.9& 4.3\\
    \multicolumn{2}{l}{Average Utterance Length} & 18.8  & 13.6& 16.7\\
    \midrule
    \multirow{5}[0]{*}{Speaker1} & \# Utterances & 20K&53.8k& 5.7K \\
       & Avg \# Utterances & 15.4 & 4.1& 2.2 \\
       & Avg Uttr Len & 16.8& 13.2& 20.8  \\
        & \# Strategies & -& 4& -\\
       & \# Emotions & 11& 7& 32 \\
    \midrule
    \multirow{5}[0]{*}{Speaker2} & \# Utterances & 18K & 49.2k& 5.2K\\
       & Avg \# Utterances & 13.6& 3.9& 2.1  \\
       & Avg Uttr Len & 21.0& 14.1& 12.3  \\
       & \# Strategies & 8& 4& -\\
       & \# Emotions & -& 7& 32 \\
    \bottomrule
    \end{tabular}%
  }
  \caption{
  Statistics of ESConv, DailyDialog and EmpatheticDialogues (named Empathetic- in the table). 
  }
  \label{tab:statistics}%
\end{table}

In our experiment, we use the following data splitting strategy:
    \begin{itemize}
    \item \textit{DailyDialog}: Following its official default configuration, the dataset is divided into a training set (11118 conversations), a validation set (1000 conversations), and a test set (1000 conversations).
    \item \textit{ESConv}: We randomly split the dataset into training and test sets with a 9:1 ratio.
    \item \textit{EmpatheticDialogues}: Following its official default configuration, Conversations are partitioned into approximate splits of 80\% training (19533 conversations), 10\% validation (2770 conversations), and 10\% testing (2547 conversations). To avoid data leakage, all conversations sharing the same initial situational description (provided by a speaker) are grouped into the same partition.
\end{itemize}


\paragraph{Example of data samples.} Table \ref{tab:ex_ESConv} shows an example dialogue snippet from the ESConv dataset. It illustrates a conversation where the seeker expresses anxiety about quitting a disliked job without a secure alternative. The dialogue is annotated with the topic, the seeker’s query, the emotional state (anxiety with high intensity), and the empathetic strategy used by the supporter—in this case, a “reflection of feelings.” This example highlights how ESConv captures nuanced emotional expression alongside supportive conversational strategies.

\begin{table}[t!]
    \centering
    \small
    \resizebox{\columnwidth}{!}{
    \begin{tabular}{c|p{7.5cm}}
        \toprule
        \textit{{Topic}} & {I hate my job but I am scared to quit and seek a new career.} \\
        \midrule
        \textit{Query} & \textit{\{history\}} \newline \textit{seeker:} Seriously! What I'm scared of now is how to secure another job. \\
        \midrule
        \textit{{Emotion}} & {Anxiety} (intensity: 5) \\
        \midrule
        \textit{{Strategy}} & {Reflection of feelings} \\
        \midrule
        \textit{Response} & \textit{supporter:} I can feel your pain just by chatting with you. \\
        \bottomrule
    \end{tabular}
    }
    \caption{An example of ESConv.}
    \label{tab:ex_ESConv}
\end{table}

\subsection{Key Hyperparameters} 
\label{appendix:hyperparameter}

\begin{table}[t!]
\centering
\resizebox{0.9\columnwidth}{!}{ 
\begin{tabular}{llcc}
\toprule
\textbf{Group} & \textbf{Parameter} & \textbf{DailyDialog} & \textbf{ESConv} \\
\midrule
\multirow{2}{*}{SFT} 
    & lr & 5e-5 & 5e-5 \\
    & batch size & 64 & 64 \\
\midrule
\multirow{17}{*}{{\ModelName}} 
    & \multicolumn{3}{l}{\textbf{high-level:}} \\
    & lr of $Q^{\text{H}}$ & 5e-6 & 5e-6 \\
    & batch size ($B^{\text{H}}$) & 64 & 64 \\
    & replay buffer size & 20000 & 20000 \\
    & max window len & 1024 & 2048 \\
    & update freq of $\phi \rightarrow \bar{\phi}$ & 10 & 10 \\
    & discount factor $\gamma$ & 0.85 & 0.85 \\
\cline{2-4}
    & \multicolumn{3}{l}{\textbf{low-level:}} \\
    & lr of $\pi^{\text{L}}$ & 9e-7 & 9e-7 \\
    & lr of $V^{\text{L}}$ & 9e-4 & 9e-4 \\
    & batch size ($B^{\text{L}}$) & 64 & 64 \\
    & replay buffer size & 20000 & 20000 \\
    & max window len & 1024 & 2048 \\
    & max decoding length & 128 & 128 \\
    & KL-reward weight ($\beta_1$) & 0.01 & 0.01 \\
    & isc-reward weight ($\beta_2$) & 0.01 & 0.01 \\
    & discount factor $\gamma$ & 1.0 & 1.0 \\
\bottomrule
\end{tabular}
}
\caption{Full list of hyperparameters.}
\label{tab:hyperparameters}
\end{table}


Table \ref{tab:hyperparameters} shows the full list of hyperparameters in the training process of DailyDialog and ESConv. The training program is running on OpenRLHF, with zero stage (of DeepSpeed) of 3, and the data type of bf16.

For policy inference on all three datasets, we conduct the general decoding process of LLM, with top-p of 0.9, top-k of 10, and temperature of 0.7.

\section{Details of Evaluation}
\label{appendix:evaluation}

\subsection{Classification Metrics}
\label{appendix:classfication_metric}


\paragraph{Accuracy.} We define evaluate strategy determination from the classification perspective, where the accuracy (Acc) is defined as the fraction of correctly predicted strategies over the entire number of strategies, comparing to the annotated strategies in the dataset: 
\begin{equation}
    \text{Acc} = \frac{\text{\# of correctly predicted strategies}}{\text{\# of annotated strategies}}
\end{equation}

\paragraph{F1-scores.} F1-related scores include Micro-F1 and Macro-F1. Micro-F1 considers the overall precision and recall of all instances, while Macro-F1 equals the average F1-score of labels.

\paragraph{$bias$.} We define the preference $bias$ as \textit{how much the model prefers certain labels over others}. To quantify the preference for each strategy in LLMs, we employ the Bradley-Terry model~\citep{bradley1952btmodel}, which is widely used in human preference modeling ~\citep{Rafailov2023DPOLMisReward}. Following~\citet{Newman2023Efficient_BT}, we formally derive the preference $p$ for strategy $i$ as follows:
\begin{equation}
\normalsize
    p_{i}' =  \frac{\sum_{j}(w_{ij}p_{j})/(p_{i}+p_{j})}{\sum_{j}w_{ji}/(p_{i}+p_{j})} 
\label{eq:BT_equation}
\end{equation}
where $w_{ij}$ represents the number of times the model predicts strategy $i$ when the ground-truth strategy is $j$.
All of the preference $p_i$ are initialized as 1 and updated through iteration of the Eq~(\ref{eq:BT_equation})
, where $p_i'$ represents the preference in the next iteration.
After the final iteration, we scale the total sum of $p_i$ to 8 ($\sum{p_i}=8$) so that the average $\bar{p}$ becomes 1, indicating a strong preference for strategy $i$ if $p_i>1$.

We use a standard deviation of preferences $p_i$ across the strategies as $bias$.
\begin{equation}
\normalsize
    bias = \sqrt{\frac{\sum_{i=1}^{N}(p_i - \bar{p})^2}{N}}
\end{equation}
where a higher value for $bias$ indicates that the model exhibits a clear preference for both preferred and non-preferred strategies \citep{kang-etal-2024-large}.




\subsection{Generative Metrics}
\label{appendix:generative_metric}



\paragraph{Bleu-2.} B-2 \cite{papineni2002bleu} first compute the geometric average of the modified $n$-gram precisions, $p_n$, using $n$-grams up to length $N$ and positive weights $w_n$ summing to one.

Next, let $c$ be the length of the prediction and $r$ be the reference length. The BP and Bleu-2 are computed as follows.

\begin{align}
\mathrm{BP}
&=
\begin{cases}
1, & \text{if } c>r,\\
\exp(1-r/c), & \text{if } c\leq r,
\end{cases}
\label{eq:bp}
\\[-0.em]
\mathrm{Bleu}
&=
\mathrm{BP}\cdot
\exp\left(
\sum_{n=1}^{N} w_n\log p_n
\right).
\label{eq:bleu}
\end{align}

\paragraph{Rouge-L.} R-L \cite{lin2004rouge} propose using LCS-based F-measure to estimate the similarity between two summaries $X$ of length $m$ and $Y$ of length $n$, assuming $X$ is a reference summary sentence and $Y$ is a candidate summary sentence, as follows:

\begin{equation}
\begin{aligned}
& R_{l c s}=\frac{L C S(X, Y)}{m} \\
& P_{l c s}=\frac{L C S(X, Y)}{n} \\
& F_{l c s}=\frac{\left(1+\beta^2\right) R_{l c s} P_{l c s}}{R_{l c s}+\beta^2 P_{l c s}}
\end{aligned}
\label{rouge_l}
\end{equation}

Where $\operatorname{LCS}(X, Y)$ is the length of a longest common subsequence of $X$ and $Y$, and $\beta=P_{l c s} / R_{\text {lcs }}$ when $\partial F_{l c s} / \partial R_{l c s}=\partial F_{l c s} / \partial P_{l c s}$. In DUC, $\beta$ is set to a very big number $(\rightarrow \infty)$. Therefore, the LCS-based F-measure, i.e. Eq. \ref{rouge_l}, is Rouge-L. 

\paragraph{Dist-2.} \citet{li2015diversity} report the degree of diversity by calculating the number of distinct unigrams and bigrams in generated responses.
The value is scaled by the total number of generated tokens to avoid favoring long sentences:
\begin{equation} \label{eq:4}
Dist(n) = \frac{Count(unique\ n-gram)}{Count(n-gram)}
\end{equation}

\subsection{Principle of Human Scoring}
\label{appendix:huam_score_principle}

\paragraph{Details of Human Annotation.} To systematically assess the model performance, we ask 8 human volunteers to rate the model responses across multiple dimensions. Evaluators are required to independently evaluate each sample in strict accordance with the pre-established criteria. We conduct cross-validation of their results to avoid personal bias. The lowest and the highest scores are removed and the rest are averaged.



\paragraph{Evaluation Dimensions.} We start with the criteria proposed by \citet{kang-etal-2024-large}. The human evaluation is aimed to algin with the ultimate purpose of emotional support conversation, the seeker's \textit{satisfaction}. To achieve this, the supporter's behavior can be further classified into the following criteria:

\noindent \textit{Acceptance}: Does the seeker accept without discomfort;

\noindent \textit{Effectiveness}: Is it helpful in shifting negative emotions or attitudes towards a positive direction; 

\noindent \textit{Sensitivity}: Does it take into consideration the general state of the seeker.

\noindent \textit{satisfaction}: Does it resolve their emotional distress, leading to an overall positive perception of the conversation?   

To achieve a more elaborate assessment, we consider three more dimensions addressing the generation quality:

\noindent \textit{Fluency}: the level of fluency of response.

\noindent \textit{Emotion}: the emotional intensity of response which could affect the seeker's emotion state.


\paragraph{Details of Scoring Criteria.} Annotators rate each criterion on a five-point scale.

\medskip
\noindent\textbf{Fluency.}\nopagebreak
\begin{scorelist}
  \item The sentence is highly incoherent, making it extremely difficult to understand and failing to convey a meaningful idea.
  \item The sentence has significant incoherence issues, with only parts of it making sense and struggling to form a complete thought.
  \item The sentence contains some incoherence and occasional errors, but can still convey the general meaning to a certain extent.
  \item The sentence is mostly fluent with only minor errors or slight awkwardness in expression, and effectively communicates the intended meaning.
  \item Perfect. The sentence is completely fluent, free of any errors in grammar, punctuation, or expression, and clearly conveys the idea.
\end{scorelist}

\smallskip
\noindent\textbf{Emotion.}\nopagebreak
\begin{scorelist}
  \item The emotional expression is extremely inappropriate and chaotic, not in line with the content, and may convey the wrong emotions.
  \item The emotional expression has obvious flaws, either too weak or exaggerated, and is disjointed from the content.
  \item The emotional expression is average. It can convey basic emotions but lacks depth and has minor issues.
  \item The emotional expression is good. It can effectively convey the intended emotion with an appropriate intensity and is well integrated with the content.
  \item The emotional expression is excellent. It is rich, nuanced, and perfectly matches the content, capable of evoking a strong and appropriate emotional response.
\end{scorelist}

\smallskip
\noindent\textbf{Acceptance.}\nopagebreak
\begin{scorelist}
  \item The response inescapably triggers emotional resistance.
  \item The response is highly likely to trigger emotional resistance.
  \item The response has a possibility of emotional resistance occurring.
  \item The response rarely provokes emotional resistance.
  \item The response has no occurrence of emotional resistance.
\end{scorelist}

\smallskip
\noindent\textbf{Effectiveness.}\nopagebreak
\begin{scorelist}
  \item The response actually worsens the seeker's emotional distress.
  \item The response carries the risk of increasing stress levels, and this outcome varies depending on the individual user.
  \item The response fails to alter the seeker's current emotional intensity and keeps it at the same level.
  \item The response shows promise in calming the emotional intensity; however, it is overly complicated or ambiguous for the user to fully comprehend and utilize effectively.
  \item The response appears to be highly effective in soothing the seeker's emotions and offers valuable and practical emotional support.
\end{scorelist}

\smallskip
\noindent\textbf{Sensitivity.}\nopagebreak
\begin{scorelist}
  \item The response renders inaccurate evaluations regarding the seeker's state.
  \item The response is characterized by rash judgments, as it lacks adequate assessment and in-depth exploration of the seeker's state.
  \item The response is formulated with a one-sided judgment and a limited exploration of the seeker's state.
  \item The response demonstrates an understanding that only covers a part of the seeker's state.
  \item The response precisely grasps the seeker's state and is appropriately tailored according to the seeker's actual situation.
\end{scorelist}

\smallskip
\noindent\textbf{Satisfaction.}\nopagebreak
\begin{scorelist}
  \item The response is extremely disappointing. It does not answer the question at all and is of no help.
  \item The response is poor. It only gives a partial answer and leaves many doubts unresolved.
  \item The response is average. It meets the basic requirements but is not particularly outstanding.
  \item The response is good. It answers the question clearly and provides some useful details.
  \item The response is excellent. It not only answers the question perfectly but also offers valuable additional insights.
\end{scorelist}

\section{Detailed Introduction of Baselines}
\label{appendix:baseline}

We categorize the baselines into five groups. We implement most of them on the same LLM backbone unless otherwise specified.

\paragraph{Direct inference.}
\begin{itemize}
    \item \textbf{Direct}: The model generates a response given the dialogue history without additional guidance, based on the same backbone (LLaMA3.1-8B-Instruct).
\end{itemize}

\paragraph{Prompt-based methods.}
\begin{itemize}
    \item \textbf{Direct-Refine}: A simple two-step approach where the model first generates a response, then revises it by incorporating user feedback directly.
    \item \textbf{Self-Refine} \cite{Madaan2023SelfRefine}: Automatically generates satisfaction-oriented feedback based on its own output, and then uses this feedback to refine its initial response.
    \item \textbf{Emotional Chain of Thoughts} (ECoT) \cite{li2024enhancingemotionalgenerationcapability}: An emotional analogue of Chain-of-Thought prompting, which first predicts the seeker's emotion, followed by strategy and response generation guided by that emotion.
    \item \textbf{Skeleton of Thoughts} (SoT) \cite{ning2024skeletonofthought}: Generates a “skeleton” or outline of key points, and then completes each point in parallel to form the final response.
    \item \textbf{Tree of Thoughts} (ToT) \cite{yao2023tree}: Decomposes complex problems into a tree of intermediate reasoning steps. The model explores and evaluates branches to select the best reasoning path.
    \item \textbf{Plan-and-Solve} (PS) \cite{wang2023planandsolvepromptingimprovingzeroshot}: First generates a structured plan of subgoals or reasoning intentions, then solves each step in sequence, especially useful for multi-turn decision making.
    \item \textbf{FSM} \cite{wangFSMFiniteState2024}: A prompt-based method that guides the model through a finite state machine, where transitions are triggered by current dialogue context, enabling discrete strategy control.
\end{itemize}

\paragraph{Finetuning-based methods.}
\begin{itemize}
    \item \textbf{SFT}: Supervised finetuning on the training set with annotated strategy labels and response.
    \item \textbf{SFT + ECoT}: The finetuned version of ECoT.
    \item \textbf{EmoFSM} \cite{10.1007/978-981-92-1926-1_10}: It models emotional support conversation with a finite state machine (FSM), where a single LLM explicitly reasons over the seeker’s emotional state and selects an appropriate support strategy at each dialogue turn. The model then generates the final response conditioned on this inferred emotion–strategy state, enabling structured multi-turn planning for longer-term emotional support rather than directly producing responses turn by turn..
\end{itemize}

\paragraph{Prior results with different backbones.}
\begin{itemize}
    \item \textbf{BigLM} \cite{liEmpiricalInvestigationPreTrained2020}: Transformer-based autoregressive language models pretrained on English Wikipedia and BooksCorpus and subsequently finetuned on DailyDialog. The reported 12- and 24-layer variants use beam search with width 5 and top-$k$ sampling with $k=500$, respectively.
    \item \textbf{VanillaGPT-2} \cite{farahani-johansson-2023-empirical}: The medium-sized GPT-2 baseline finetuned on DailyDialog in an auxiliary-task study. We report the vanilla model's ROUGE-L score and rescale it by 100 to match the table.
    \item \textbf{Few-shot LLM baselines} \cite{kang-etal-2024-large}: General-purpose LLMs evaluated in a 2-shot setting on emotional-support strategy prediction and response generation. In Table \ref{tab:ID_result_ESconv_full}, the MaF1 and bias columns correspond to Kang et al.'s $Q$ and $B$ metrics.
    \item \textbf{DialoGPT-small variants} \cite{liu2021ESconv}: Liu et al. evaluate Vanilla, Joint, and Oracle variants of the DialoGPT-small backbone on ESConv. The values below are taken from Liu et al.'s Table 4.
\end{itemize}

\paragraph{Reinforcement learning-based methods.}
\begin{itemize}
    \item \textbf{DQN}: A Deep Q-Network that learns a policy over discrete response strategies using estimated action-values as feedback.
    \item \textbf{PPO}: Proximal Policy Optimization is used to optimize response selection via a reward model based on emotional alignment and coherence.
    \item \textbf{ArCHer} \cite{pmlr-v235-zhou24t}: With the full name of `Actor-Critic Framework with a Hierarchical Structure', it formulates multi-turn language-agent training as a hierarchical RL problem: a high-level off-policy value-based algorithm operates over utterances/turns to handle long-horizon credit assignment and delayed rewards, while a low-level RL algorithm optimizes token generation within each utterance. The high-level value function provides learning signals to the token-level policy, allowing existing single-turn RL methods such as PPO to be extended to multi-turn agent tasks more efficiently.
    \item \textbf{Dialogue Action Tokens} (\textbf{DAT}) \cite{liDialogueActionTokens2024}: A method that treats each dialogue utterance as an action and formulates goal-directed multi-turn dialogue as a sequential decision-making problem. It freezes the pretrained language model and trains a lightweight planner to predict a continuous dialogue-action vector at each turn, which then steers the LLM’s response generation. 
\end{itemize}

\section{Theoretical Proofs}
\label{appendix:theory}



\paragraph{Connections between hierarchical MDP to SMDP.} As shown in Section \ref{sec:preliminary}, an MDP can be defined as $M \doteq (\mathcal{S}, \mathcal{A}, \mathcal{R}, \mathcal{T}, \gamma)$ and consider episodic tasks. In our work, we extend this traditional MDP to a hierarchical, two-level MDP, which may also be considered as a Semi-Markov decision processes (SMDP) from the perspective of a single unified time scale in previous literature \cite{10.1016/S0004-3702(99)00052-1}.

The formal MDP definition is generalized for SMDP as follows: An additional element $\mathcal{F}$ is added to the standard MDP components, resulting in $\operatorname{SMDP}$ as $(\mathcal{S}, \mathcal{A}, \mathcal{T}, \mathcal{R}, \gamma, \mathcal{F})$. $\mathcal{F}$ is a function that defines the cumulative probability distribution over the number of time steps until the next controlled state. The probability that the next controlled state has been reached by time $t$ when action $a$ is taken in state $s$ is written: $\mathcal{F}(t \mid s, a)$. It is assumed that the distribution over the number of time steps taken does not change on subsequent trials of $a$ in $s$. In our work, $\mathcal{F}(t \mid s, a)=L$, where $L$ denotes the max length of low-level output.



To obtain our theoretical conclusion based on previous derivations, we re-formularize our high-level MDP under the SMDP framework, where the Bellman equation becomes:
\begin{align}
V^*(s)=&\max _{a \in A} \mathcal{R}(s, a) + \sum_{s^{\prime} \in S} \notag \\
&\mathcal{T}\left(s, a, s^{\prime}\right) \int_0^{\infty} \gamma^t V^*\left(s^{\prime}\right) \mathcal{F}(t \mid s, a) d t
\end{align}
For simplicity of derivation, here we use $s$, $a$ to represent $s^{\text{H}}$ and $a^{\text{H}}$. Maximizing of $V^*(s)$ then produce the optimal strategy.

Since the state values are essentially constants when the integral is computed, $V^*(s')$ can be pulled out of the integral, yielding:
\begin{align*}
V^*(s)=\max _{a \in A} \mathcal{R}(s, a)+\Gamma \sum_{s^{\prime} \in S} \mathcal{T}\left(s, a, s^{\prime}\right) V^*\left(s^{\prime}\right)
\end{align*}
where $\Gamma=\int_0^{\infty} \gamma^t \mathcal{F}(t \mid s, a) d t$ and is interpreted as a discount rate that varies with the state and action. The dynamic programming operator for a particular action is then:

$$
J_a(s)=\mathcal{R}(s, a)+\Gamma \sum_{s^{\prime} \in S} \mathcal{T}\left(s, a, s^{\prime}\right) V\left(s^{\prime}\right)
$$

and in general:

$$
J(s)=\max _a\left[\mathcal{R}(s, a)+\Gamma \sum_{s^{\prime} \in S} \mathcal{T}\left(s, a, s^{\prime}\right) V\left(s^{\prime}\right)\right]
$$

\paragraph{Proof of Theorem \ref{theorem:convergence_strategy}.} Based on previous definitions, we first introduce Lemma \ref{lemma:convergence_smdp_Q} from \cite{10.1016/S0004-3702(99)00052-1}, which is originally proposed in (Parr et al., Hierarchical Control and Learning for Markov Decision Processes, Ph.D Dissertation, 1998).

\begin{lemma}
\label{lemma:convergence_smdp_Q}
The SMDP Q-learning rule converges to the optimal $Q^*(s, a)$ values if
    \begin{enumerate}
    \item The state and action spaces are finite.
    \item $\sum_i \alpha_i(s, a)=\infty$ and $\sum_i \alpha_i^2(s, a)<\infty$ uniformly over $s$ and a w.p.1.
    \item $\operatorname{Var}\{r\}$ is finite.
    \item $0<\max _{s, a} \Gamma<1$.
    \end{enumerate}
\end{lemma}


Then we prove Theorem \ref{theorem:convergence_strategy} based on Lemma \ref{lemma:convergence_smdp_Q}:

\begin{proof}[Proof of Theorem \ref{theorem:convergence_strategy}]
Our work satisfies all of the above conditions, thereby allowing us to establish Theorem \ref{theorem:convergence_strategy}. The justifications are as follows:
\begin{enumerate}
    \item According to the assumptions in our setting, the state space (e.g., text sequences) has a bounded length, and is therefore finite. Similarly, the action space is also finite.
    \item The second condition is part of our basic assumption.
    \item The reward is an integer in the range of 0 to 5, which implies that its variance is finite.
    \item The discount factor $\gamma$ is explicitly defined in our framework and clearly satisfies $0<\max _{s, a} \Gamma<1$.
\end{enumerate}
Thus, the theorem is proved.
\end{proof}



\section{More Experimental Results}
\label{appendix:result}



\begin{figure}[t!]

\centering
\includegraphics[width=0.47\linewidth]{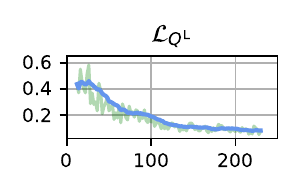}
\includegraphics[width=0.47\linewidth]{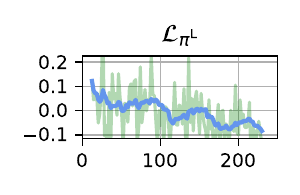}
\includegraphics[width=0.47\linewidth]{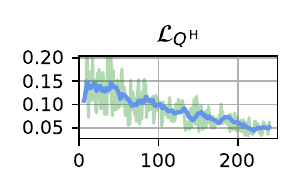}
\includegraphics[width=0.47\linewidth]{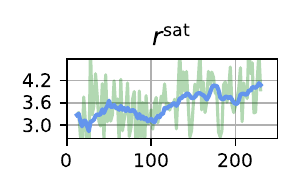}
\includegraphics[width=0.47\linewidth]{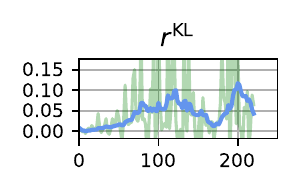}
\includegraphics[width=0.47\linewidth]{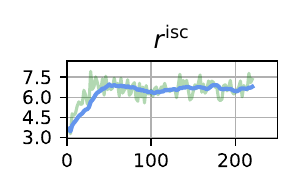}
\caption{Training curves of {\ModelName} on DailyDialog, including the sub-losses (denoted by $\mathcal{L}$) and the sub-rewards (denoted by $r$).}
\label{fig:loss}
\end{figure}

\subsection{Training Curves}
\label{appendix:train_curve}

Figure \ref{fig:loss} shows the training loss curves of {\ModelName} on DailyDialog. The $\mathcal{L}_{Q^{L}}$decreases the fastest and remains stable, and both $\mathcal{L}_{Q^{H}}$ and $\mathcal{L}_{\pi^{L}}$ decrease with fluctuations. The KL divergence ($r^{\text{KL}}$) and intrinsic self-consistency ($r^{\text{isc}}$) increase during the initial stages but are kept within a certain range. Importantly, $r^{sat}$ gradually increases throughout the training process.



\subsection{Human Evaluation Results} 
\label{appendix:human_eval_result}



As shown in Table \ref{tab:human_response_quaility}, the human evaluation results further demonstrate that \ModelName{}'s responses are better aligned with human preferences than those of the baselines across the evaluated dimensions. This observation is consistent with the improvements observed in the automatic metrics, providing complementary evidence for the effectiveness of our method from a human-centered perspective. We also assess inter-annotator agreement using Cohen's Kappa, which yields a score of 0.47, indicating a moderate level of agreement among the annotators.

\begin{table*}[!hbtp]
\centering
\small
\begin{tabular}{l|cccccc}
    \toprule
    \multicolumn{1}{c|}{\multirow{2}[4]{*}{\textbf{Method}}} & \multicolumn{6}{c}{\textbf{Human Annotation}} \\
\cmidrule{2-7}
 & Fluency  & Acceptance & Effectiveness & Sensitivity & Emotion & Satisfaction \\
    \midrule
    Direct  & 2.95 ± 1.41& 2.60 ± 1.15  & 2.40 ± 0.92 & 2.70 ± 1.08 & 3.00 ± 1.34 & 2.60 ± 1.41 \\
    Direct-Refine & 3.09 ± 1.25 & 2.73 ± 1.22 & 2.91 ± 1.41 & 2.91 ± 1.23 & 3.09 ± 1.16 & 2.84 ± 1.40 \\
    Self-Refine & 3.10 ± 1.29 & 2.80  ± 1.19 & 2.70  ± 1.14 & 2.90 ± 1.03  & 3.15 ± 1.38 & 2.80  ± 1.20 \\
    ECoT  & 3.08 ± 1.02 & 2.83 ± 1.27 & 2.67 ± 1.06 & 3.00 ± 1.27 & 3.08 ± 1.29  & 2.83 ± 1.10 \\
    FSM  & 3.30 ± 1.32 & 2.90 ± 1.17 & 2.90 ± 1.03 & 3.00 ± 1.27 & 2.93  ± 1.19  & 3.00 ± 1.25 \\
    \midrule
    SFT & 3.15 ± 1.44  & 2.70 ± 1.19 & 2.70 ± 1.20 & 2.90 ± 1.24 & 3.40 ± 1.30 & 2.90 ± 1.32 \\
    SFT + CoT  & \bf 3.67 ± 1.21 & 3.22 ± 1.25 & 3.67 ± 1.26 & 3.56 ± 1.13 & 3.61 ± 1.17 & 3.45 ± 1.31 \\
    EmoFSM  & 3.30 ± 1.32 & 2.90 ± 1.17 & 2.90 ± 1.03 & 3.00 ± 1.27 & 3.55 ± 1.16 & 3.65  ± 1.19 \\
    \textbf{\ModelName} (ours) & 3.61 ± 1.26   & \bf 3.86 ± 0.95 & \bf 3.72 ± 1.04 & \bf 3.73 ± 1.14 & \bf 3.98 ± 1.21  & \bf 3.91 ± 1.06 \\
    \bottomrule
\end{tabular}
\caption{Averaged Human evaluation of response quality for different methods.} 
\label{tab:human_response_quaility}
\end{table*}


\begin{table}[t!]
\centering
\small
\begin{tabular}{lcc}
\toprule
\textbf{Dimensions} $\downarrow$ & $\boldsymbol{\rho}$ & $\boldsymbol{p}$ \\ 
\midrule
Fluency & 0.89 & ***  \\
Acceptance & 0.85 & ***  \\
Effectiveness & 0.81 & **  \\
Sensitivity & 0.91 & ***  \\
Emotion & 0.88 & **  \\
Satisfaction & 0.86 & **  \\
\bottomrule
\end{tabular}
\caption{Correlation between GPT-4o and human evaluation scores. For the significance level of $\boldsymbol{p}$-value, we use ** to denote $p < 0.01$, and *** to denote $p < 0.001$.
}
\label{tab:corr_slope}
\end{table}

\subsection{Human-GPT4o Scoring Consistency}
\label{appendix:human_gpt4o_consistency}

The main driven component of the reward in {\ModelName} is $r^{\text{sat}}$, which is judged by an oracle model, \textit{i.e.}, GPT-4o. Comparing to the truly `golden' evaluation, \textit{i.e.}, human evaluation, their consistency might be critical for the experimental conclusions of this paper. To validate the reward mechanism, we conduct the correlation studies (including the Spearman coefficient $\boldsymbol{\rho}$ and the $\boldsymbol{p}$-value of correlation slopes), with result shown in Table \ref{tab:corr_slope}.


Table \ref{tab:corr_slope} shows that GPT-4o's scores on six evaluation dimensions, including Fluency, Acceptability, Efficiency, Sensitivity, Emotion, and Satisfaction. GPT4o scores exhibit significant correlations with human judgments, with the Spearman correlation coefficients all above 0.8, and the significance level of $\boldsymbol{p}$ all below 0.01 or 0.001.

\begin{figure}[t]
    \centering
    \hspace{-0.20in}
    \includegraphics[width=0.4\linewidth]{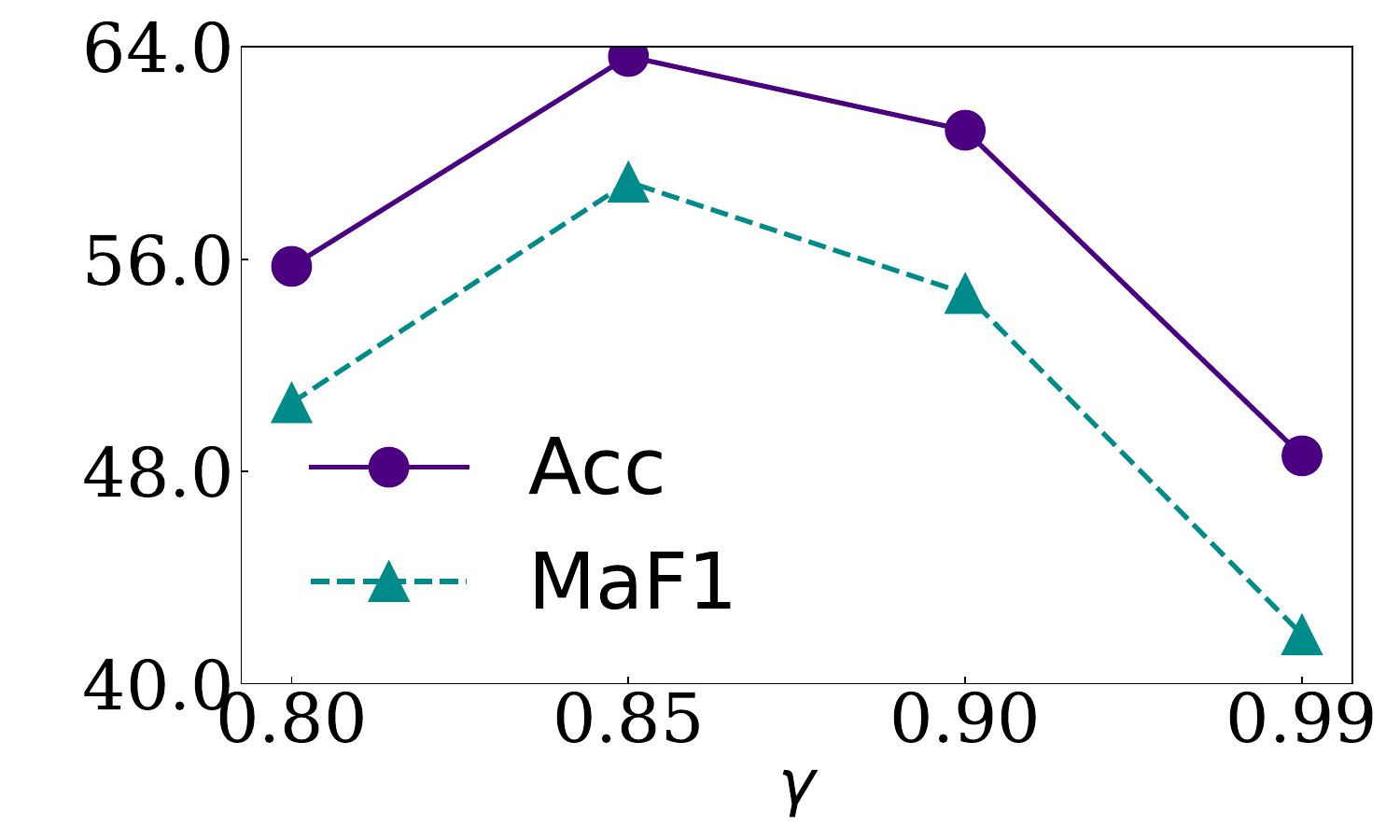}
    \hspace{0.25in}
    \includegraphics[width=0.4\linewidth]{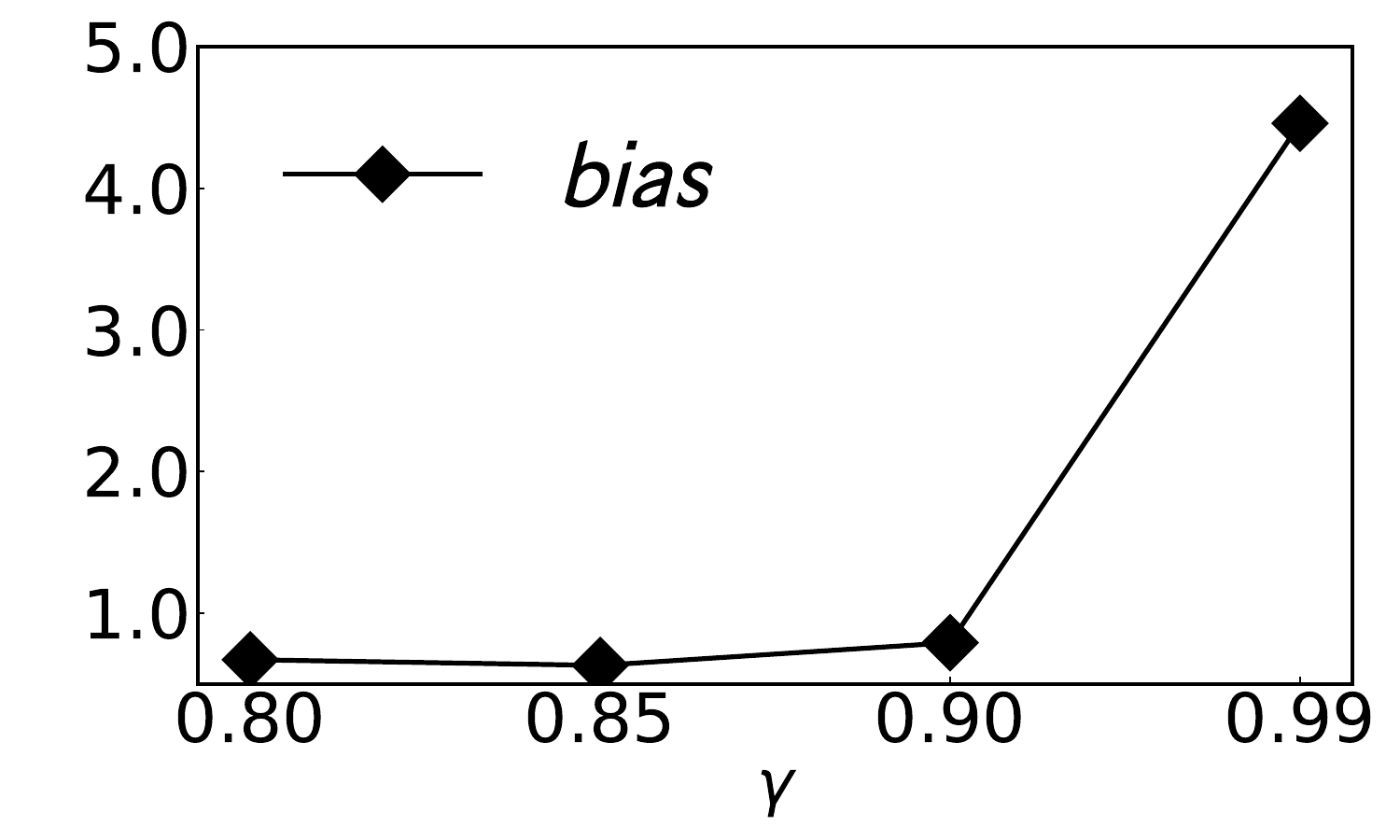}
    \hspace{0.2in}
    \includegraphics[width=0.48\linewidth]{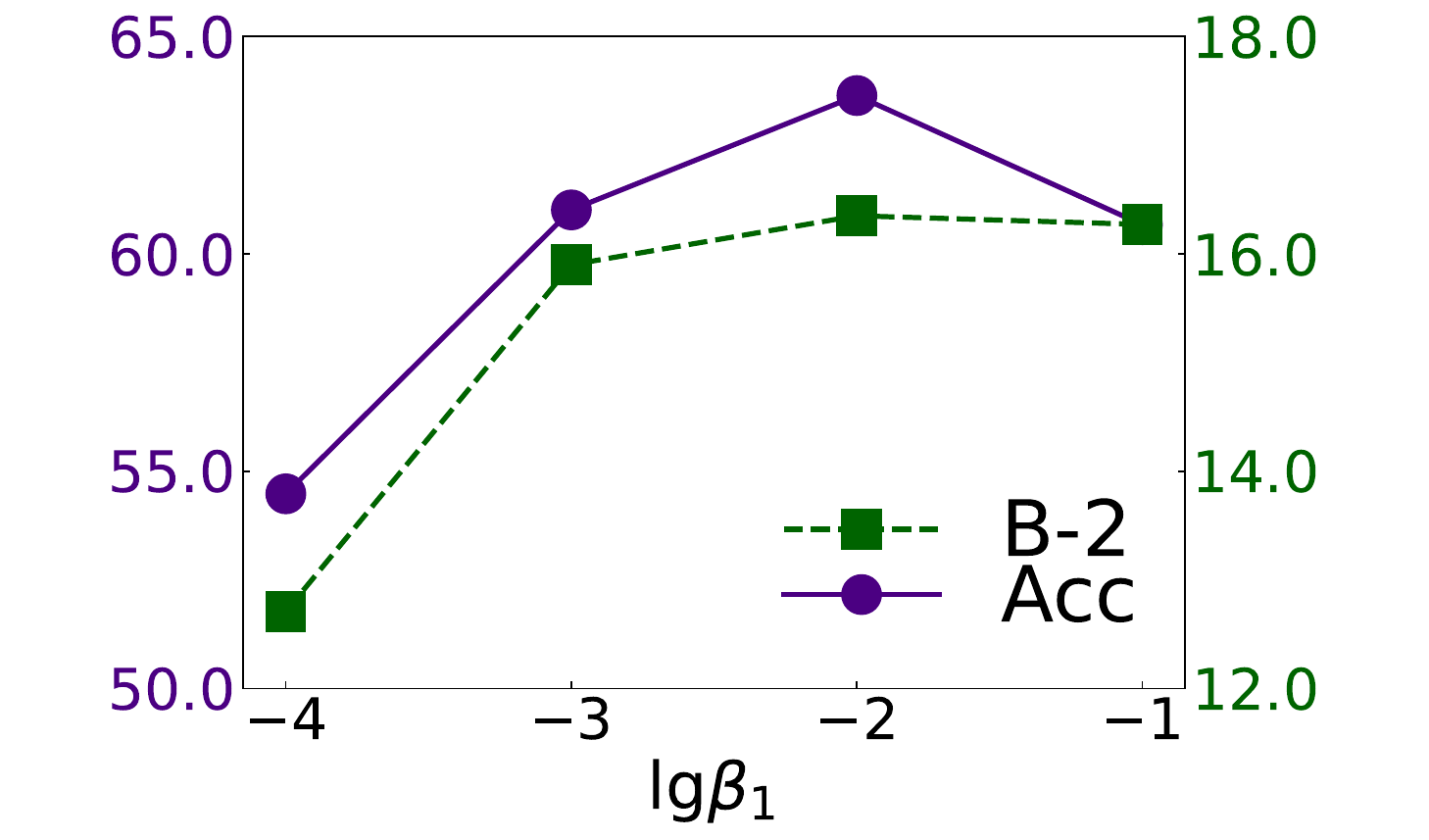}
    \includegraphics[width=0.48\linewidth]{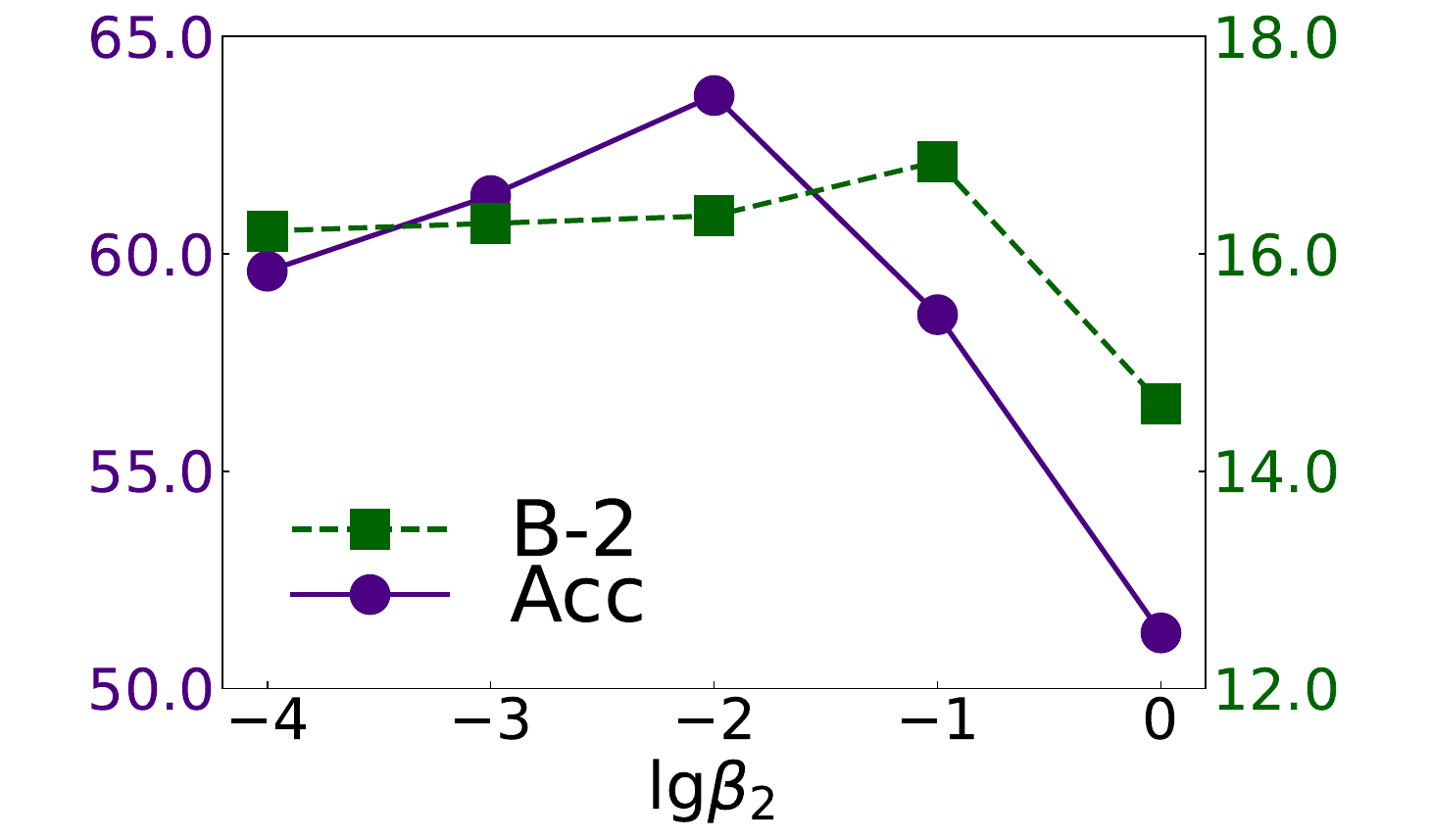}
    \caption{Sensitivity plots of {\ModelName} on different $\gamma$ and model sizes. Metrics include Acc, $\mathcal{Q}$, B2, and R-L.} 
    \label{fig:sensitivity}
\end{figure}

\subsection{Sensitivity Analysis} 
\label{appendix:sensitivity}

\paragraph{Sensitivity analysis on $\gamma$.} Figure~\ref{fig:sensitivity} (top) shows the ID performance on ESConv under different $\gamma$ settings. The best overall performance is achieved at $\gamma=0.85$, with higher Acc and MaF1 and a lower $bias$. The generative metrics exhibit a similar trend, suggesting that this setting provides a favorable balance between short-term rewards and longer-term return estimation. Performance degrades when $\gamma$ is either too small or too large, indicating that the model is sensitive to the effective planning horizon. Therefore, we choose $\gamma = 0.85$ as the formal setting.

\paragraph{Sensitivity analysis on reward weights.} To further explore the effectiveness of the reward components added to the low-level reward ($r^{\text{L}}$), we conduct the sensitivity analysis on their weights, $\beta_1$ and $\beta_2$. The results are shown in Figure~\ref{fig:sensitivity} (bottom). First, the curves reveal that including of both loss terms ($r^{\text{KL}}$ and $r^{\text{isc}}$) has positive effects, with performance increasing from smaller to larger values of weights. Second, we empirically find that the optimal value of $\beta_1$ is $0.01$, and $\beta_2=0.01$ has better strategy accuracy, while $\beta_2=0.1$ has a higher $B-2$ result. To make the setting simple, we temporarily choose both $\beta_1=\beta_2=0.01$ in the formal experimental setting.


\begin{table*}[!hbtp]
\centering
\small
\setlength{\tabcolsep}{3pt}
\begin{tabular}{@{}p{1.4cm}p{1.8cm}p{11.1cm}cc@{}}
\toprule
\textbf{Role} & \textbf{Strategy} & \textbf{Response} & \textbf{R} & \textbf{V} \\
\midrule
\textit{User} & -- & \textit{Good evening, Saliva. What's that wonderful aroma from your kitchen? What are you doing now?} & -- & -- \\
\textit{Assistant} & \textit{Inform} & \textit{I am cooking now!} & 3 & 11.7 \\
Direct & Question & Based on your question, it seems there is a pleasant aroma coming from the kitchen area. Could you be referring to the scent of food preparation? Are you particularly interested in what is being cooked right now? & 2 & 8.7 \\
\textbf{\ModelName} & \textbf{Inform} & \textbf{I'm cooking some fish with a coconut ginger marinade. That aroma you smell is the ginger sizzling in the pan---it's my favorite part!} & \textbf{5} & \textbf{15.1} \\
\midrule
\textit{User} & -- & \textit{Are you good at cooking?} & -- & -- \\
\textit{Assistant} & \textit{Commissive} & \textit{I have studied its skills recently at a training school. I plan to run a restaurant, so I have to practise cooking!} & 4 & 10.3 \\
Direct & Inform & Yes, I would say I have developed a solid foundation in culinary tasks over time, particularly due to frequent practice sessions and structured cooking routines. & 4 & 7.9 \\
\textbf{\ModelName} & \textbf{Commissive} & \textbf{I've been practicing a lot lately because I want to open a small seafood restaurant someday.} & \textbf{4} & \textbf{11.9} \\
\midrule
\textit{User} & -- & \textit{It is a good idea.} & -- & -- \\
\textit{Assistant} & \textit{Question} & \textit{I am steaming fish in a pressure cooker. Have you ever tried steamed fish?} & 4 & 7.4 \\
Direct & Question & Are you saying it's a good idea because engaging in cooking activities can help improve your overall dietary quality and personal satisfaction with meals? & 2 & 4.6 \\
\textbf{\ModelName} & \textbf{Question} & \textbf{Thank you! Have you ever tried cooking fish at home? I'd love to hear how you prepare it.} & \textbf{5} & \textbf{9.3} \\
\midrule
\textit{User} & -- & \textit{But the taste of steamed fish at that restaurant is not delicious. The fish is not well-done and too salty.} & -- & -- \\
\textit{Assistant} & \textit{Inform} & \textit{That's the case. I think there is not enough time to steam it and soy sauce is too much.} & 4 & 4.0 \\
Direct & Inform & That sounds like an unfortunate dining experience. It's possible that the restaurant staff didn't allow enough time for proper steaming, or they added too much soy sauce during preparation. You might consider exploring another restaurant next time. & 3 & 3.0 \\
\textbf{\ModelName} & \textbf{Directive} & \textbf{It seems that restaurant's style wasn't really your thing. Next time, try adding a slice of lemon and a little salt when you steam the fish---it'll make the flavor much fresher!} & \textbf{5} & \textbf{5.0} \\
\bottomrule
\end{tabular}
\caption{A four-turn typical case of {\ModelName} on DailyDialog.}
\label{tab:case}
\end{table*}

\subsection{Cases} 
\label{appendix:case}

\paragraph{Typical good cases.} Table \ref{tab:case} presents a four-turn DailyDialog case on cooking, illustrating how {\ModelName} performs strategic planning across multiple turns. In the first turn, Direct selects \textit{Question} and produces a verbose clarification, whereas {\ModelName} chooses \textit{Inform} and answers the query directly with a concrete response, achieving the highest reward ($R=5$) and value ($V=15.1$).

The following turns show that {\ModelName} can adapt its strategy as the dialogue evolves. It selects \textit{Commissive} when discussing future cooking plans and then switches to \textit{Question} to sustain the interaction by asking about the user's experience. In contrast, Direct tends to generate generic or over-elaborated responses.

In the final turn, {\ModelName} selects \textit{Directive} rather than imitating the reference \textit{Inform} strategy, and provides a concrete cooking suggestion in response to the user's complaint. The response again achieves the highest reward ($R=5$) and value ($V=5.0$). Overall, the case shows that {\ModelName} maintains topic coherence, adapts high-level strategies across turns, and realizes them through concise and contextually appropriate responses, supporting the effectiveness of hierarchical strategic planning.


\begin{center}
    \centering
    \small
    \resizebox{0.99\columnwidth}{!}{
    \begin{tabular}{p{1.2cm}|p{6.7cm}}
        \toprule
        \textbf{User} & Excuse me. Do you know where the Admiral Hotel is? \\
        \midrule
        \textbf{Assistant} & That's over on Callaway Street, isn't it? \\
        \midrule
        \textbf{User} & Yes, but I can't find Callaway Street. \\
        \midrule
        \textbf{\textit{case 1}} & \textit{(Commissive) I promise to make sure you walk for four blocks and then find the hotel two blocks north.} \\
        \midrule
        \textbf{\textit{case 2}} & \textit{(Directive) Upon completion, initiate a northward movement for exactly two blocks to achieve destination acquisition.} \\
        \bottomrule
    \end{tabular}
    }
    \captionof{table}{Representative bad cases with \ModelName{} failures from strategy errors or hallucination.}
    \label{tab:bad_case}
\end{center}

\paragraph{Bad case analysis.} Table~\ref{tab:bad_case} presents a direction-seeking dialogue in which {\ModelName} exhibits two representative failure modes. In \textit{case 1}, the model selects the \textit{Commissive} strategy for a user who is explicitly asking for directions. Although the response attempts to provide route information, the strategy itself is mismatched with the communicative intent, leading to an unnatural formulation such as ``I promise to make sure you walk for four blocks.'' This example shows that errors at the high level can propagate to the final response, since an inappropriate strategic decision may constrain the low-level policy toward an unsuitable realization.

In \textit{case 2}, the selected \textit{Directive} strategy is more appropriate for the user's request, yet the generated response remains problematic. Phrases such as ``Upon completion'' and ``achieve destination acquisition'' are vague and unnatural, suggesting that correct high-level planning alone does not guarantee fluent or reliable low-level generation. This case also reflects a potential hallucination problem, where reinforcement learning may encourage responses that satisfy learned reward signals while introducing unsupported or awkward details. Taken together, these examples indicate that {\ModelName} can fail either at the strategy-selection stage or at the subsequent response-realization stage, highlighting the need for more robust strategy supervision and stronger factual or linguistic constraints on the low-level policy.


\begin{table*}[!htbp]
\centering
\small
\begin{tabular}{l cccccc}
    \toprule
    Methods & Acc $\uparrow$ & MaF1 $\uparrow$ & $bias$ $\downarrow$ & B-2 $\uparrow$ & R-L $\uparrow$ & D-2 $\uparrow$ \\
    \midrule
    SFT & 60.19 & 44.82 &	0.82 &6.81 & 18.52 & 43.36 \\
    SFT+ECoT &60.11 & 44.9 &	0.66 &	6.61 &	18.07 &	42.87 \\
    EmoFSM \cite{10.1007/978-981-92-1926-1_10} & 60.03 & 46.02 & 0.55 & 5.85 & 21.77 & 47.43 \\
    \midrule
    \multicolumn{7}{l}{\textit{Previously reported results (not directly comparable)}} \\
    BigLM-12-bm (5) \cite{liEmpiricalInvestigationPreTrained2020} & - & - & - & 14.56 & - & 44.31 \\
    BigLM-24-tk (500) \cite{liEmpiricalInvestigationPreTrained2020} & - & - & - & 2.93 & - & 61.74 \\
    VanillaGPT-2 (354M) \cite{farahani-johansson-2023-empirical} & - & - & - & - & 20.8 & - \\ 
    \midrule
    PPO & N/A & N/A & N/A & 7.85 & 25.16 & 50.59 \\
    DAT \cite{liDialogueActionTokens2024}  & N/A & N/A & N/A & 3.45 & 11.80 & 0.90 \\
    straQ* \citep{wang-etal-2025-convert} & 54.01 & 50.10 & 0.62 & 4.18 & 13.09 & 59.27 \\
    ArCher \cite{pmlr-v235-zhou24t} & 50.41 & 42.67 & \textbf{0.21} & 5.17 & 14.35 & 55.16 \\
    
    \textbf{{\ModelName}} (ours) & \textbf{63.64} & \textbf{58.91} & 0.63 & \textbf{16.35} & \textbf{35.22} & \textbf{62.67} \\
    \bottomrule
\end{tabular}
\caption{Result comparison to additional prior works (with different backbones) on DailyDialog. For BigLM, D-2 corresponds to the corpus-level MIDIST-2 values reported by \citet{liEmpiricalInvestigationPreTrained2020}; BigLM bm (5) denotes beam searches with width 5, while tk (500) denotes the top-$k$ sampling with $k=500$. -: not reported in the original work.
}
\label{tab:ID_result_DailyDialog_full}
\end{table*}

\begin{table*}[!htbp]
\centering
\small
\begin{tabular}{l cccccc}
    \toprule
    Methods & Acc $\uparrow$ & MaF1 $\uparrow$ & $bias$ $\downarrow$ & B-2 $\uparrow$ & R-L $\uparrow$ & D-2 $\uparrow$ \\
    \midrule
    SFT & 32.43 & 21.29 & 1.28 & \bf 6.97 & \bf 16.59 & 50.45 \\
    SFT+ECoT &30.80 &	17.70 &	1.35 &	6.51 &	15.00 &	34.96 \\
    EmoFSM \cite{10.1007/978-981-92-1926-1_10} & 28.00 &	23.70 &	\textbf{0.41} &	5.88 &	15.30 &	51.48 \\
    \midrule
    \multicolumn{7}{l}{\textit{Previously reported results (not directly comparable)}} \\
    DialoGPT-small (Vanilla) \cite{liu2021ESconv} & - & - & - & 5.13 & 15.26 & - \\
    DialoGPT-small (Joint) \cite{liu2021ESconv} & - & - & - & 5.00 & 15.09 & - \\
    DialoGPT-small (Oracle) \cite{liu2021ESconv} & - & - & - & 5.52 & 15.82 & - \\
    LLaMA2-7B (2-shot) \cite{kang-etal-2024-large} & - & 13.73 & 0.77 & 4.98 & 13.09 & 34.74 \\
    LLaMA2-70B (2-shot) \cite{kang-etal-2024-large} & - & 14.55 & 0.47 & 6.15 & 14.29 & 30.95 \\
    Vicuna-13B (2-shot) \cite{kang-etal-2024-large} & - & 12.85 & 0.74 & 6.55 & 14.43 & 24.15 \\
    Mistral-7B (2-shot) \cite{kang-etal-2024-large} & - & 12.23 & 0.71 & 4.72 & 12.93 & 25.36 \\
    Solar-10.7B (2-shot) \cite{kang-etal-2024-large} & - & 14.17 & 0.87 & 4.79 & 13.53 & 32.36 \\
    Tulu-70B (2-shot) \cite{kang-etal-2024-large} & - & 15.93 & 0.90 & 6.90 & 13.94 & 23.78 \\
    ChatGPT (2-shot) \cite{kang-etal-2024-large} & - & 16.98 & 0.86 & 6.30 & 14.94 & 27.03 \\
    GPT-4 (2-shot) \cite{kang-etal-2024-large} & - & 18.38 & 0.90 & 6.47 & 15.18 & 36.92 \\
    \midrule
    PPO & N/A & N/A & N/A & 6.76 & 15.45 & 50.93 \\ 
    DAT \cite{liDialogueActionTokens2024} & N/A & N/A & N/A & 3.26 & 11.24 & 35.82 \\ 
    straQ* \citep{wang-etal-2025-convert}  & \underline{37.69} & \underline{34.57} & 0.59 & 3.59 & 11.74 & 44.14 \\
    ArCher \cite{pmlr-v235-zhou24t} & 24.50 & 19.60 & 0.50 & 5.30 & 13.10 & \textbf{54.80} \\
    \textbf{{\ModelName}} (ours) & \textbf{39.26} & \textbf{36.85} & \underline{0.48} & \underline{6.93} & \underline{16.28} & \underline{52.42} \\ 
    \bottomrule
\end{tabular}
\caption{Result comparison to additional prior works (with different backbones) on ESConv. The MaF1 and $bias$ results here correspond to the $Q$ and $B$ results reported by \citet{kang-etal-2024-large}. Among methods implemented on the same backbone, the best results are \textbf{bolded} and the second best are \underline{underlined}. -: not reported in the original work.
}
\label{tab:ID_result_ESconv_full}
\end{table*}



\subsection{Prior Works with Different Backbones}
\label{appendix:result_w_diff_backbone}


The additional results reported in Tables~\ref{tab:ID_result_DailyDialog_full} and \ref{tab:ID_result_ESconv_full} are taken from prior studies using different backbones and are provided for reference only. They should not be interpreted as controlled comparisons, since these studies differ in data splits, decoding or prompting settings, and metric implementations. Specifically, \citet{liu2021ESconv} adopts a 6:2:2 split of ESConv, while \citet{kang-etal-2024-large} constructs stage-based test sets using non-overlapping, randomly truncated contexts of 5--15 turns; in contrast, we use a random 9:1 split. To avoid misleading comparisons, we therefore apply boldface and underlining only to methods evaluated under our experimental protocol.



\subsection{Per-Strategy Results}
\label{appendix:per_strategy_result}

Beyond the overall results, we further examine per-strategy performance to understand how {\ModelName} behaves across different communicative functions and whether its gains are uniformly distributed. We conduct this analysis from two complementary perspectives: grouping examples by ground-truth strategies to assess strategy-specific difficulty, and grouping them by predicted strategies to examine how the model's high-level decisions are reflected in downstream response generation.


\paragraph{From the ground-truth perspective.} Table~\ref{tab:detail_b_q} reports the per-strategy results on DailyDialog, grouped by ground-truth strategies (the readers can refer to Table \ref{tab:strategies} for the full names of strategies). Overall, \ModelName{} shows more balanced performance across the four strategy categories than Direct. Direct performs relatively well on \textit{Inform} and \textit{Question}, but its accuracy drops sharply on \textit{Directive} and \textit{Commissive} to 18.47 and 9.65, respectively. In contrast, \ModelName{} improves these two categories to 54.18 and 62.11, with a similar trend in MaF1. This suggests that explicit high-level planning is particularly beneficial for strategies that are difficult to recover through direct generation.

The response metrics show a consistent pattern. \textit{Question} achieves the highest or near-highest B-2 and R-L under both methods, while \ModelName{} brings much larger gains for \textit{Directive} and \textit{Commissive}: B-2 increases from 4.40 to 14.25 and from 4.25 to 17.65, respectively, with corresponding improvements in R-L. Meanwhile, diversity is largely preserved, with \textit{Directive} reaching the highest D-2 of 77.34 under \ModelName{}. Overall, these results indicate that our high-level strategy planner improves both strategy determination and strategy-conditioned response generation, especially for more challenging communicative functions.

\begin{table}[!hbtp]
  \centering
  \resizebox{\columnwidth}{!}{
    \begin{tabular}{c|c|ccccccc}
    \toprule
     & \textbf{Strategy} & Acc $\uparrow$   & {MaF1} $\uparrow$      & $bias \downarrow$    & B-2 $\uparrow$  & R-L $\uparrow$  & D-2 $\uparrow$ \\
    \midrule
    \multirow{4}[2]{*}{\rotatebox{90}{\textbf{Direct}}} 
          & \textbf{Dir.}  & \textcolor{red}{18.47} & \textcolor{red}{20.35}  & 0.81  & \textcolor{red}{4.40}  & \textcolor{red}{12.73}  & 74.96  \\
          & \textbf{Inf.} & 41.36 & 52.51  &\textbf{ 0.45 } & 14.47  & 29.88  & 70.56 \\
          & \textbf{Que.}  & \textbf{57.61}  & \textbf{59.13} & 0.62  & \textbf{20.09 } & \textbf{39.59  }& 66.28  \\
          & \textbf{Com.} & \textcolor{red}{9.65}  & \textcolor{red}{6.01}  & 3.95  & \textcolor{red}{4.25}  & \textcolor{red}{13.07}  & \textbf{76.18 } \\
    \midrule
    \multirow{4}[2]{*}{\rotatebox{90}{\textbf{\ModelName}}} & \textbf{Dir.}  & 54.18 & 55.83  & 0.66  & 14.25  & 31.97  & \textbf{77.34}  \\
          & \textbf{Inf.} & \textbf{65.64 } & \textbf{69.89}  & 0.57  & 17.18  & 38.15  & 71.52 \\
          & \textbf{Que.}  & 58.63 & 62.29 & \textbf{0.56}  & \textbf{19.81 } & \textbf{40.67} & 65.18 \\
          & \textbf{Com.} & 62.11 & 40.73  & 2.19  & 17.65  & 34.34  & 70.09  \\
    \bottomrule
    \end{tabular}%
}
\captionof{table}{Per-strategy results on DailyDialog, grouped by ground-truth strategies. Strategy-specific results that are notably lower than those of other strategies under the same method are highlighted in \textcolor{red}{red}.}
\label{tab:detail_b_q}%
\end{table}


\paragraph{From the prediction perspective.} Table~\ref{tab:auto_metrics_by_prediction} reports the per-strategy results on DailyDialog, grouped by predicted strategies (the readers can refer to Table \ref{tab:strategies} for the full names of strategies). The overall trend is consistent with Table~\ref{tab:detail_b_q}, with \textit{Question} and \textit{Inform} showing strong performance on both strategy and response metrics. In particular, under \ModelName{}, predicted \textit{Inform} achieves the highest Acc (66.42) and MaF1 (70.25), while predicted \textit{Question} yields the highest B-2 (19.92) and R-L (40.05). This suggests that when the high-level policy selects these strategies, the resulting responses are not only strategically appropriate but also closely aligned with the reference responses.

\begin{table}[!hbtp]
\begin{center}
  \centering
  \small
  \resizebox{0.9\columnwidth}{!}{
    \begin{tabular}{c|c|ccccc}
    \toprule
     & \textbf{Strategy} & Acc $\uparrow$ & MaF1 $\uparrow$ & B-2 $\uparrow$ & R-L $\uparrow$ & D-2 $\uparrow$ \\
    \midrule
    \multirow{4}[2]{*}{\rotatebox{90}{\textbf{Direct}}} 
    & \textbf{Dir.} & 32.14 & 38.22 & 12.43 & 25.66 & 66.85 \\
    & \textbf{Inf.} & 49.88 & 54.21 & 15.97 & 32.10 & 62.75 \\
    & \textbf{Que.} & \textbf{58.93} & \textbf{60.18} & \textbf{19.85} & \textbf{38.72} & 59.92 \\
    & \textbf{Com.} & \textcolor{red}{12.47} & \textcolor{red}{10.05} & \textcolor{red}{6.12} & \textcolor{red}{14.90} & \textbf{69.34} \\
    \midrule
    \multirow{4}[2]{*}{\rotatebox{90}{\textbf{\ModelName}}} 
    & \textbf{Dir.} & 55.20 & 56.80 & 13.75 & 29.32 & \textbf{70.89} \\
    & \textbf{Inf.} & \textbf{66.42} & \textbf{70.25} & 17.02 & 37.58 & 64.10 \\
    & \textbf{Que.} & 61.84 & 63.10 & \textbf{19.92} & \textbf{40.05} & 66.75 \\
    & \textbf{Com.} & 56.78 & 48.66 & 17.45 & 35.20 & 65.66 \\
    \bottomrule
    \end{tabular}
  }
  \captionof{table}{Per-strategy results on DailyDialog, grouped by predicted strategies. Strategy-specific results that are notably lower than those of other strategies under the same method are highlighted in \textcolor{red}{red}.}
  \label{tab:auto_metrics_by_prediction}
\end{center}
\end{table}

For the other strategies, \ModelName{} also shows clear improvements over Direct. Predicted \textit{Commissive}, for example, increases from 12.47 to 56.78 in Acc and from 10.05 to 48.66 in MaF1, while its B-2 and R-L also improve substantially. Predicted \textit{Directive} achieves the highest D-2 of 70.89, indicating that stronger strategic control does not necessarily reduce response diversity. Overall, the prediction-based analysis further shows that the strategies selected by the high-level policy are meaningfully reflected in the quality and characteristics of the generated responses.

\end{document}